\documentclass{article}

\usepackage[preprint]{neurips_2026}
\usepackage[utf8]{inputenc}
\usepackage[T1]{fontenc}
\usepackage{hyperref}
\usepackage{url}
\usepackage{booktabs}
\usepackage{amsfonts}
\usepackage{nicefrac}
\usepackage{microtype}
\usepackage{xcolor}
\usepackage{graphicx}
\usepackage{amsmath}
\usepackage{amssymb}
\usepackage{amsthm}
\usepackage{multirow}
\usepackage{makecell}
\usepackage{subcaption}
\usepackage{array}
\usepackage{enumitem}
\usepackage{placeins}
\usepackage{float}
\usepackage[ruled,vlined,linesnumbered]{algorithm2e}

\makeatletter
\IfFileExists{\jobname.aux}{}{\IfFileExists{bibcite_preamble.tex}{\bibcite{alvarez2018towards}{{1}{2018}{{Alvarez~Melis and Jaakkola}}{{}}}
\bibcite{bandi2018detection}{{2}{2018}{{Bandi et~al.}}{{Bandi, Geessink, Manson, Van~Dijk, Balkenhol, Hermsen, Bejnordi, Lee, Paeng, Zhong, et~al.}}}
\bibcite{brancati2022bracs}{{3}{2022}{{Brancati et~al.}}{{Brancati, Anniciello, Pati, Riccio, Scognamiglio, Jaume, De~Pietro, Di~Bonito, Foncubierta, Botti, et~al.}}}
\bibcite{bulten2022artificial}{{4}{2022}{{Bulten et~al.}}{{Bulten, Kartasalo, Chen, Str{\"o}m, Pinckaers, Nagpal, Cai, Steiner, Van~Boven, Vink, et~al.}}}
\bibcite{bunch1978free}{{5}{1978}{{Bunch}}{{}}}
\bibcite{buzzard2024paths}{{6}{2024}{{Buzzard et~al.}}{{Buzzard, Hemker, Simidjievski, and Jamnik}}}
\bibcite{campanella2019clinical}{{7}{2019}{{Campanella et~al.}}{{Campanella, Hanna, Geneslaw, Miraflor, Werneck Krauss~Silva, Busam, Brogi, Reuter, Klimstra, and Fuchs}}}
\bibcite{caron2021emerging}{{8}{2021}{{Caron et~al.}}{{Caron, Touvron, Misra, J{\'e}gou, Mairal, Bojanowski, and Joulin}}}
\bibcite{chen2019looks}{{9}{2019}{{Chen et~al.}}{{Chen, Li, Tao, Barnett, Rudin, and Su}}}
\bibcite{chen2021annotation}{{10}{2021{a}}{{Chen et~al.}}{{Chen, Chen, Yu, Chen, Chang, Hsu, Hsiao, Yeh, and Chen}}}
\bibcite{chen2018learning}{{11}{2018}{{Chen et~al.}}{{Chen, Song, Wainwright, and Jordan}}}
\bibcite{chen2021whole}{{12}{2021{b}}{{Chen et~al.}}{{Chen, Lu, Shaban, Chen, Chen, Williamson, and Mahmood}}}
\bibcite{chen2024towards}{{13}{2024}{{Chen et~al.}}{{Chen, Ding, Lu, Williamson, Jaume, Song, Chen, Zhang, Shao, Shaban, et~al.}}}
\bibcite{chen2020simple}{{14}{2020}{{Chen et~al.}}{{Chen, Kornblith, Norouzi, and Hinton}}}
\bibcite{cheng2021computational}{{15}{2021}{{Cheng et~al.}}{{Cheng, Liu, Huang, Hong, Wang, Zhan, Han, Ni, Huang, and Zhang}}}
\bibcite{cover1999elements}{{16}{1999}{{Cover}}{{}}}
\bibcite{deyoung2020eraser}{{17}{2020}{{DeYoung et~al.}}{{DeYoung, Jain, Rajani, Lehman, Xiong, Socher, and Wallace}}}
\bibcite{dietterich1997solving}{{18}{1997}{{Dietterich et~al.}}{{Dietterich, Lathrop, and Lozano-P{\'e}rez}}}
\bibcite{ding2025multimodal}{{19}{2025}{{Ding et~al.}}{{Ding, Wagner, Song, Chen, Lu, Zhang, Vaidya, Jaume, Shaban, Kim, et~al.}}}
\bibcite{dong2025fast}{{20}{2025}{{Dong et~al.}}{{Dong, Jiang, Jiang, Li, and Zhang}}}
\bibcite{du2025rethinking}{{21}{2025}{{Du et~al.}}{{Du, Mao, Lu, Qi, Zhang, Gu, and Jiao}}}
\bibcite{ehteshami2017diagnostic}{{22}{2017}{{Ehteshami~Bejnordi et~al.}}{{Ehteshami~Bejnordi, Veta, Johannes~van Diest, Van~Ginneken, Karssemeijer, Litjens, Van Der~Laak, consortium, Hermsen, Manson, et~al.}}}
\bibcite{fong2017interpretable}{{23}{2017}{{Fong and Vedaldi}}{{}}}
\bibcite{fourkioticamil}{{24}{2024}{{Fourkioti et~al.}}{{Fourkioti, {De Vries}, and Bakal}}}
\bibcite{guo2023higt}{{25}{2023}{{Guo et~al.}}{{Guo, Zhao, Wang, and Yu}}}
\bibcite{he2016deep}{{26}{2016}{{He et~al.}}{{He, Zhang, Ren, and Sun}}}
\bibcite{he2020momentum}{{27}{2020}{{He et~al.}}{{He, Fan, Wu, Xie, and Girshick}}}
\bibcite{hinton2015distilling}{{28}{2015}{{Hinton et~al.}}{{Hinton, Vinyals, and Dean}}}
\bibcite{ilse2018attention}{{29}{2018}{{Ilse et~al.}}{{Ilse, Tomczak, and Welling}}}
\bibcite{jang2017categorical}{{30}{2017}{{Jang et~al.}}{{Jang, Gu, and Poole}}}
\bibcite{kang2023benchmarking}{{31}{2023}{{Kang et~al.}}{{Kang, Song, Park, Yoo, and Pereira}}}
\bibcite{kapse2025gecko}{{32}{2025}{{Kapse et~al.}}{{Kapse, Pati, Yellapragada, Das, Gupta, Saltz, Samaras, and Prasanna}}}
\bibcite{kim2018interpretability}{{33}{2018}{{Kim et~al.}}{{Kim, Wattenberg, Gilmer, Cai, Wexler, Viegas, et~al.}}}
\bibcite{kingma2014adam}{{34}{2014}{{Kingma and Ba}}{{}}}
\bibcite{koh2020concept}{{35}{2020}{{Koh et~al.}}{{Koh, Nguyen, Tang, Mussmann, Pierson, Kim, and Liang}}}
\bibcite{kraus2016classifying}{{36}{2016}{{Kraus et~al.}}{{Kraus, Ba, and Frey}}}
\bibcite{li2021dual}{{37}{2021}{{Li et~al.}}{{Li, Li, and Eliceiri}}}
\bibcite{li2019attention}{{38}{2019}{{Li et~al.}}{{Li, Li, Gertych, Knudsen, Speier, and Arnold}}}
\bibcite{lin2023interventional}{{39}{2023}{{Lin et~al.}}{{Lin, Yu, Hu, Xu, and Chen}}}
\bibcite{louizos2018learning}{{40}{2018}{{Louizos et~al.}}{{Louizos, Welling, and Kingma}}}
\bibcite{lu2021data}{{41}{2021}{{Lu et~al.}}{{Lu, Williamson, Chen, Chen, Barbieri, and Mahmood}}}
\bibcite{lu2023visual}{{42}{2023}{{Lu et~al.}}{{Lu, Chen, Zhang, Williamson, Chen, Ding, Le, Chuang, and Mahmood}}}
\bibcite{maddison2017concrete}{{43}{2017}{{Maddison et~al.}}{{Maddison, Mnih, and Teh}}}
\bibcite{maron1997framework}{{44}{1997}{{Maron and Lozano-P{\'e}rez}}{{}}}
\bibcite{martins2016softmax}{{45}{2016}{{Martins and Astudillo}}{{}}}
\bibcite{miller1969froc}{{46}{1969}{{Miller}}{{}}}
\bibcite{nemhauser1978analysis}{{47}{1978}{{Nemhauser et~al.}}{{Nemhauser, Wolsey, and Fisher}}}
\bibcite{oquab2024dinov2}{{48}{2024}{{Oquab et~al.}}{{Oquab, Darcet, Moutakanni, Vo, Szafraniec, Khalidov, Fernandez, Haziza, Massa, El-Nouby, et~al.}}}
\bibcite{pantanowitz2011review}{{49}{2011}{{Pantanowitz et~al.}}{{Pantanowitz, Valenstein, Evans, Kaplan, Pfeifer, Wilbur, Collins, and Colgan}}}
\bibcite{Petsiuk2018rise}{{50}{2018}{{Petsiuk et~al.}}{{Petsiuk, Das, and Saenko}}}
\bibcite{qu2024rethinking}{{51}{2024}{{Qu et~al.}}{{Qu, Ma, Luo, Guo, Wang, and Song}}}
\bibcite{russakovsky2015imagenet}{{52}{2015}{{Russakovsky et~al.}}{{Russakovsky, Deng, Su, Krause, Satheesh, Ma, Huang, Karpathy, Khosla, Bernstein, et~al.}}}
\bibcite{shannon1948mathematical}{{53}{1948}{{Shannon}}{{}}}
\bibcite{shao2025multiple}{{54}{2025}{{Shao et~al.}}{{Shao, Chen, Song, Runevic, Lu, Ding, and Mahmood}}}
\bibcite{shao2021transmil}{{55}{2021}{{Shao et~al.}}{{Shao, Bian, Chen, Wang, Zhang, Ji, et~al.}}}
\bibcite{simonyan2013deep}{{56}{2013}{{Simonyan et~al.}}{{Simonyan, Vedaldi, and Zisserman}}}
\bibcite{song2023artificial}{{57}{2023}{{Song et~al.}}{{Song, Jaume, Williamson, Lu, Vaidya, Miller, and Mahmood}}}
\bibcite{srivastava2014dropout}{{58}{2014}{{Srivastava et~al.}}{{Srivastava, Hinton, Krizhevsky, Sutskever, and Salakhutdinov}}}
\bibcite{sundararajan2017axiomatic}{{59}{2017}{{Sundararajan et~al.}}{{Sundararajan, Taly, and Yan}}}
\bibcite{tang2023multiple}{{60}{2023}{{Tang et~al.}}{{Tang, Huang, Zhang, Zhou, Zhang, and Liu}}}
\bibcite{tarvainen2017mean}{{61}{2017}{{Tarvainen and Valpola}}{{}}}
\bibcite{tourniaire2023ms}{{62}{2023}{{Tourniaire et~al.}}{{Tourniaire, Ilie, Hofman, Ayache, and Delingette}}}
\bibcite{tsallis1988possible}{{63}{1988}{{Tsallis}}{{}}}
\bibcite{van2008visualizing}{{64}{2008}{{Van~der Maaten and Hinton}}{{}}}
\bibcite{verghese2023computational}{{65}{2023}{{Verghese et~al.}}{{Verghese, Lennerz, Ruta, Ng, Thavaraj, Siziopikou, Naidoo, Rane, Salgado, Pinder, et~al.}}}
\bibcite{wang2019weakly}{{66}{2019}{{Wang et~al.}}{{Wang, Chen, Gan, Lin, Dou, Tsougenis, Huang, Cai, and Heng}}}
\bibcite{weinstein2013cancer}{{67}{2013}{{Weinstein et~al.}}{{Weinstein, Collisson, Mills, Shaw, Ozenberger, Ellrott, Shmulevich, Sander, and Stuart}}}
\bibcite{xiong2021nystromformer}{{68}{2021}{{Xiong et~al.}}{{Xiong, Zeng, Chakraborty, Tan, Fung, Li, and Singh}}}
\bibcite{xu2014weakly}{{69}{2014}{{Xu et~al.}}{{Xu, Zhu, Eric, Chang, Lai, and Tu}}}
\bibcite{yang2024mambamil}{{70}{2024}{{Yang et~al.}}{{Yang, Wang, and Chen}}}
\bibcite{yao2020whole}{{71}{2020}{{Yao et~al.}}{{Yao, Zhu, Jonnagaddala, Hawkins, and Huang}}}
\bibcite{ye2026asmil}{{72}{2026}{{Ye et~al.}}{{Ye, Hamidi, Chi, Li, Pilanci, Ogawa, Haseyama, and Plataniotis}}}
\bibcite{yoon2018invase}{{73}{2018}{{Yoon et~al.}}{{Yoon, Jordon, and Van~der Schaar}}}
\bibcite{zeiler2014visualizing}{{74}{2014}{{Zeiler and Fergus}}{{}}}
\bibcite{zhang2022dtfd}{{75}{2022}{{Zhang et~al.}}{{Zhang, Meng, Zhao, Qiao, Yang, Coupland, and Zheng}}}
\bibcite{zhang2024attention}{{76}{2024}{{Zhang et~al.}}{{Zhang, Li, Sun, Zheng, Zhu, and Yang}}}
\bibcite{zhang2025aem}{{77}{2025}{{Zhang et~al.}}{{Zhang, Li, Sun, Shui, Li, Zhu, and Yang}}}
}{}}
\makeatother

\newcommand{\gce}{\textit{+GCE}}
\newcommand{\best}[1]{\textbf{#1}}
\newcommand{\second}[1]{\underline{#1}}
\newcommand{\uptri}[1]{\ensuremath{\triangle #1}}
\newcommand{\downtri}[1]{\ensuremath{\triangledown #1}}

\theoremstyle{definition}
\newtheorem{definition}{Definition}
\theoremstyle{plain}
\newtheorem{proposition}{Proposition}

\title{GCE-MIL: Faithful and Recoverable Evidence\\for Multiple Instance Learning in Whole-Slide Imaging}
\author{%
Xiangyu Li\\
College of Intelligence and Computing\\
Tianjin University\\
\texttt{xiangyuli@tju.edu.cn}
\And
Ran Su\thanks{Corresponding author.}\\
College of Intelligence and Computing\\
Tianjin University\\
\texttt{ran.su@tju.edu.cn}
}

\begin{document}
\raggedbottom
\maketitle

\begin{abstract}
Multiple instance learning (MIL) is the standard approach for whole-slide image (WSI) classification and survival prediction, where attention-based models aggregate patch features into slide-level predictions.
These models treat attention weights as evidence for their predictions, but attention is optimized for classification, not for identifying which patches actually support the diagnosis.
This conflation leads to three failures: selected patches are \emph{insufficient} (keeping them alone drops Macro-F1 by $0.078$), \emph{unnecessary} (removing them barely changes the prediction), and \emph{unrecoverable} (continuous attention scores disagree with discrete patch subsets used at inference).
The central premise is that evidence quality should be optimized directly through explicit criteria---Sufficiency, Necessity, and Recoverability (S/N/R)---rather than inherited as a byproduct of classification.
GCE-MIL is a backbone-agnostic wrapper implemented through three injection modes and three evidence components: a grounding mechanism that aligns selection with domain-specific concepts, noisy-OR coverage that acts as a differentiable proxy for interventional evidence search, and threshold-plus-repair recovery that converts continuous selectors into discrete subsets through marginal-guided repair.
Across $9$ backbones $\times$ $9$ datasets ($81$ configurations), GCE-MIL improves average Macro-F1 by $0.024$ and C-index by $0.014$, reduces the continuous-discrete gap by $4$--$7\times$, and increases complement degradation by $2$--$4\times$.
With optional tile prefiltering after discrete recovery, inference runs up to $5\times$ faster while retaining $0.989\times$ full-bag utility.
\end{abstract}

\section{Introduction}
\label{sec:intro}

Whole-slide images contain thousands of patches, diagnostically relevant tissue is sparse, and training supervision is usually available only at the slide level.
Multiple instance learning (MIL) is therefore the standard modeling choice: a slide is treated as a bag of patches, and a bag-level predictor is trained from slide labels \citep{ilse2018attention,lu2021data,shao2021transmil}.
The same models are often asked to provide evidence by visualizing attention weights or instance scores.
This reuse is convenient, but it leaves a clinical question unresolved: which patches actually support the prediction under intervention?

Current WSI MIL methods mainly improve the predictor.
ABMIL introduces gated attention pooling \citep{ilse2018attention}; CLAM adds clustering-constrained attention \citep{lu2021data}; TransMIL models inter-instance correlation with transformers \citep{shao2021transmil}; DSMIL, DTFD-MIL, IBMIL, MHIM-MIL, CAMIL, and HDMIL improve aggregation, regularization, context, or efficiency \citep{li2021dual,zhang2022dtfd,lin2023interventional,tang2023multiple,fourkioticamil,dong2025fast}.
Attention-regularized variants such as ACMIL, AEM, and ASMIL stabilize or deconcentrate attention maps \citep{zhang2024attention,zhang2025aem,ye2026asmil}.
All these methods, however, share an implicit assumption that classification accuracy is the optimization target and that attention or score rankings are an interpretable byproduct.
This conflates two distinct goals: predicting correctly and learning correct evidence.

The failure is not simply that attention chooses the wrong number of patches; in the BRACS diagnostic, WSI evidence is often structurally non-unique.
Table~\ref{tab:minimal_subset_main} reports a recursive BRACS diagnostic using an ABMIL teacher.
For $k=8$, sufficient-prefix search finds $2.2302$ disjoint sufficient subsets per slide on average, and $72.67\%$ of validation slides admit at least two such subsets.
Attention top-$k$ also finds multiple candidate subsets, but its keep-only drop remains higher ($0.0766$ at $k=8$), while random top-$k$ almost never finds reusable evidence.
Multiple tissue regions can therefore each preserve the same slide-level decision, while a single softmax ranking collapses them into one list.
This BRACS diagnostic motivates the S/N/R framework, which is then evaluated across the full nine-dataset benchmark.
Figure~\ref{fig:motivation_main} summarizes the resulting S/N/R tension.

\begin{table*}[t]
\centering
\caption{\textbf{Minimal sufficient subset analysis on BRACS validation slides using an ABMIL teacher.} A subset is sufficient if it preserves the full-bag predicted class and has probability drop at most $0.05$.}
\label{tab:minimal_subset_main}
\resizebox{\textwidth}{!}{%
\begin{tabular}{llccccccc}
\toprule
$k$ & Policy & Subsets/slide$\uparrow$ & Slides $\ge 2$ (\%)$\uparrow$ & Slides $\ge 3$ (\%)$\uparrow$ & Keep-only drop$\downarrow$ & Remove-union-1$\uparrow$ & Remove-union-2$\uparrow$ & Remove-union-3$\uparrow$ \\
\midrule
8 & Sufficient-prefix & 2.2302$\pm$0.2635 & 72.67$\pm$13.28 & 50.35$\pm$13.16 & 0.0183$\pm$0.0019 & 0.0806$\pm$0.0361 & 0.0685$\pm$0.0200 & 0.0588$\pm$0.0102 \\
8 & Attention top-$k$ & 2.1031$\pm$0.2113 & 70.89$\pm$9.30 & 39.43$\pm$12.08 & 0.0766$\pm$0.0180 & 0.1134$\pm$0.0294 & 0.1163$\pm$0.0256 & 0.1252$\pm$0.0270 \\
8 & Random top-$k$ & 0.0688$\pm$0.0411 & 0.38$\pm$0.47 & 0.00$\pm$0.00 & 0.1763$\pm$0.0198 & 0.0000$\pm$0.0000 & 0.0000$\pm$0.0001 & 0.0000$\pm$0.0001 \\
\midrule
16 & Sufficient-prefix & 1.7295$\pm$0.2188 & 51.07$\pm$10.66 & 21.88$\pm$11.42 & 0.0162$\pm$0.0025 & 0.1086$\pm$0.0343 & 0.1014$\pm$0.0323 & 0.0862$\pm$0.0200 \\
16 & Attention top-$k$ & 1.9251$\pm$0.1473 & 59.89$\pm$5.94 & 32.62$\pm$8.98 & 0.0595$\pm$0.0199 & 0.0823$\pm$0.0294 & 0.0963$\pm$0.0309 & 0.1002$\pm$0.0280 \\
16 & Random top-$k$ & 0.1054$\pm$0.0459 & 1.34$\pm$0.99 & 0.00$\pm$0.00 & 0.1728$\pm$0.0345 & 0.0001$\pm$0.0001 & 0.0003$\pm$0.0003 & 0.0003$\pm$0.0003 \\
\midrule
32 & Sufficient-prefix & 1.3682$\pm$0.0675 & 29.00$\pm$5.34 & 7.83$\pm$1.54 & 0.0113$\pm$0.0016 & 0.1085$\pm$0.0261 & 0.0779$\pm$0.0094 & 0.0687$\pm$0.0203 \\
32 & Attention top-$k$ & 1.2785$\pm$0.0536 & 22.70$\pm$4.27 & 5.15$\pm$1.75 & 0.0877$\pm$0.0225 & 0.1260$\pm$0.0343 & 0.1371$\pm$0.0337 & 0.1358$\pm$0.0318 \\
32 & Random top-$k$ & 0.2156$\pm$0.0173 & 2.86$\pm$1.20 & 0.00$\pm$0.00 & 0.1335$\pm$0.0093 & 0.0001$\pm$0.0002 & 0.0001$\pm$0.0002 & 0.0001$\pm$0.0002 \\
\bottomrule
\end{tabular}}
\end{table*}

\begin{figure*}[t]
\centering
\includegraphics[width=\textwidth]{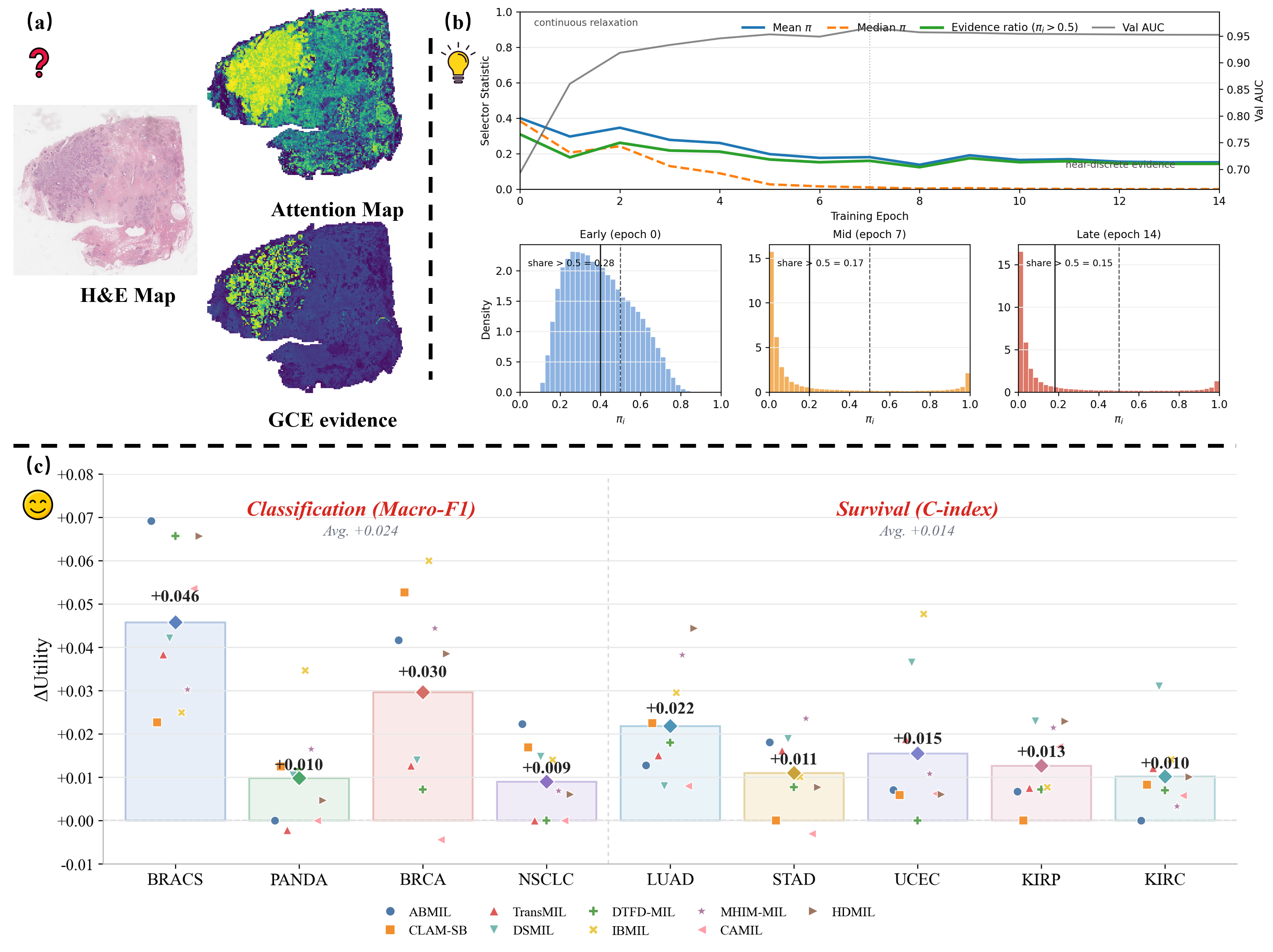}
\caption{\textbf{Three evidence failures in classification-optimized MIL.}
(Left) Attention top-$k$ achieves only $0.640$ keep-only Macro-F1 vs.\ GCE $0.722$: attention is not sufficient evidence.
(Middle) The continuous selector becomes bimodal during training, enabling discretization with C-D gap $0.004$.
(Right) Adding GCE preserves $0.99\times$ full-bag performance across backbones.}
\label{fig:motivation_main}
\end{figure*}

Evidence quality is formalized through three model-relative criteria.
\textbf{Sufficiency} asks whether the selected subset alone preserves the prediction; \textbf{Necessity} asks whether removing the subset degrades the prediction; \textbf{Recoverability} asks whether the continuous selector learned during training yields a faithful discrete subset at inference.
GCE-MIL targets these criteria directly with a semantic anchor bank, a continuous selector trained through exact noisy-OR coverage, and threshold-plus-greedy recovery.
It is a wrapper rather than a replacement backbone: the host MIL architecture remains unchanged, and the evidence gate is injected as attention-logit bias, feature reweighting, or a hybrid according to the host aggregation type.

The paper makes four contributions:
\begin{itemize}[leftmargin=1.2em,itemsep=2pt]
\item A formalization of MIL evidence quality through Sufficiency, Necessity, and Recoverability, with supporting theoretical notes on independence, recoverability, and coverage in Appendix~\ref{app:theory}.
\item A BRACS minimal-subset diagnostic showing evidence non-uniqueness in validation slides, motivating evidence selection beyond a single attention ranking.
\item GCE-MIL, a plug-in wrapper that combines semantic grounding, noisy-OR coverage, and discrete recovery through attention-bias, feature-reweighting, and hybrid injection modes.
\item An evaluation across $9$ backbones and $9$ datasets, covering prediction metrics, intervention diagnostics, localization, stability, ablations, and computational cost.
\end{itemize}

\section{Related Work}
\label{sec:related}

MIL research for WSI can be read as a sequence of improvements to bag prediction.
Pooling design is addressed by ABMIL, CLAM, DSMIL, and TransMIL \citep{ilse2018attention,lu2021data,li2021dual,shao2021transmil}; attention concentration is mitigated by ACMIL, AEM, and ASMIL \citep{zhang2024attention,zhang2025aem,ye2026asmil}; overfitting is reduced by DTFD-MIL and MHIM-MIL \citep{zhang2022dtfd,tang2023multiple}; spatial context and efficiency are modeled by CAMIL and HDMIL \citep{fourkioticamil,dong2025fast}.
These works are useful host architectures for GCE-MIL.
Their optimization target remains slide-level prediction, however, while evidence sufficiency, necessity, and recoverability are usually evaluated only after training.
GCE-MIL is therefore orthogonal: it plugs into these backbones and adds evidence objectives rather than competing as another pooling module.

Sparse selection and post-hoc attribution address explanation more directly but still leave S/N/R under-specified.
$L_0$, Concrete, and Gumbel relaxations provide differentiable gates \citep{louizos2018learning,maddison2017concrete,jang2017categorical}, yet they do not ground selected patches in pathology concepts or model multi-source diagnostic coverage.
Gradient saliency, integrated gradients, and occlusion provide post-hoc scores \citep{simonyan2013deep,sundararajan2017axiomatic,zeiler2014visualizing}, but the predictor is already fixed and the thresholded subset is not optimized to be sufficient or necessary.
GCE-MIL differs by training selection and prediction jointly, grounding gates with TITAN text anchors, and evaluating the recovered subset through explicit S/N/R interventions.

Subset and concept explanation methods provide useful context for this formulation.
L2X and INVASE learn instance- or feature-level rationales \citep{chen2018learning,yoon2018invase}, perturbation methods such as Meaningful Perturbations and RISE score regions through input interventions \citep{fong2017interpretable,Petsiuk2018rise}, and ERASER popularizes sufficiency/necessity-style rationale evaluation in NLP \citep{deyoung2020eraser}.
Concept methods such as SENN, ProtoPNet, concept bottleneck models, and TCAV connect predictions to human-readable concepts \citep{alvarez2018towards,chen2019looks,koh2020concept,kim2018interpretability}.
GCE-MIL draws on these ideas but targets the WSI MIL setting: the object being recovered is a slide-level, sparse patch subset whose continuous selector and discrete evidence are evaluated under the same bag predictor.

\section{Preliminaries and Motivation}
\label{sec:preliminaries}

\subsection{Notation and MIL Formulation}
\label{sec:mil_wsi}
In MIL for computational pathology, supervision is provided only at the slide level.
A whole-slide image is represented as a bag $X=\{x_i\}_{i=1}^{N}$ of tissue patches, where $N$ can range from hundreds to tens of thousands.
A pretrained encoder maps each patch $x_i$ to a fixed-dimensional embedding $h_i\in\mathbb{R}^{d}$ ($d=1024$ throughout).
An attention-based MIL model assigns a scalar attention score to each embedding via a learnable scorer:
\begin{equation}
z_i = f_\theta(h_i), \qquad \alpha_i = \frac{\exp(z_i)}{\sum_{j=1}^{N}\exp(z_j)},
\label{eq:attention}
\end{equation}
where $\alpha_i$ lies on the probability simplex $\Delta^N$.
The slide-level representation $h_{\mathrm{bag}} = \sum_{i=1}^{N} \alpha_i h_i$ is a convex combination of instance features weighted by the attention distribution, and is passed to a classifier to produce the bag-level prediction $\hat{y}=f(h_{\mathrm{bag}})$.

After training, the attention weights $\{\alpha_i\}$ are reused as evidence: the top-ranked patches are presented as the model's explanation for its prediction.
However, this reuse conflates two distinct objectives---classification accuracy and evidence quality---because the attention mechanism is optimized solely for the former.

\subsection{Motivation: Three Evidence Failures}
\label{sec:motivation}
Three systematic failures arise when attention is treated as evidence, motivating the S/N/R criteria formalized below.

\textbf{(P1) Insufficiency.}
Keeping only the top-attended patches should preserve the prediction if they constitute sufficient evidence.
Table~\ref{tab:causal} shows this fails: keeping attention top-$k$ drops Macro-F1 by $0.078$ from the full bag, averaged across nine datasets and nine backbones.

\textbf{(P2) Unnecessity.}
Removing the top-attended patches should degrade the prediction if they are necessary evidence.
However, removing attention top-$k$ changes Macro-F1 by only $0.033$ (Table~\ref{tab:causal}), indicating the model largely recovers from the remaining patches.

\textbf{(P3) Unrecoverability.}
During training, the selector operates in continuous space, but at inference a discrete subset must be extracted by thresholding.
The continuous-discrete gap reaches $0.029$ for ABMIL attention, compared with $0.005$--$0.011$ for GCE-wrapped backbones (Appendix Table~\ref{tab:diagnostics_bracs}), meaning the discrete inference-time evidence disagrees with the continuous signal used during training.

These failures persist across ABMIL, TransMIL, CLAM, DSMIL, and other architectures \citep{ilse2018attention,shao2021transmil,lu2021data,li2021dual}.
The problem is compounded by evidence non-uniqueness: a recursive minimal-subset diagnostic on BRACS (Table~\ref{tab:minimal_subset_main}) reveals that $72.67\%$ of slides admit at least two disjoint sufficient subsets, yet attention produces a single global ranking that conflates these sources.
This diagnostic motivates evidence selection beyond a single attention ranking; the subsequent experiments evaluate whether optimizing S/N/R improves evidence quality across datasets and backbones.

\subsection{S/N/R: Three Criteria for Evidence Quality}
\label{sec:snr}
The following definitions formalize what it means for an evidence subset to be ``correct,'' with each criterion addressing one failure.

\begin{definition}[$\delta_s$-Sufficiency, addressing P1]
For a bag predictor $f$ and subset $S\subseteq\{1,\ldots,N\}$, let $X_S=\{x_i:i\in S\}$.
$S$ is $\delta_s$-sufficient if $|f(X_S)-f(X)|\leq \delta_s$.
\end{definition}

\begin{definition}[$\delta_n$-Necessity, addressing P2]
Let $X_{\neg S}=\{x_i:i\notin S\}$.
$S$ is $\delta_n$-necessary if $|f(X_{\neg S})-f(X)|\geq \delta_n$.
\end{definition}

\begin{definition}[$\delta_r$-Recoverability, addressing P3]
For a continuous selector $\pi\in[0,1]^N$, let $X_{\pi}=\{\pi_i x_i\}$ and $S(\pi)=\{i:\pi_i\geq\tau\}$.
$\pi$ is $\delta_r$-recoverable if $|f(X_{\pi})-f(X_{S(\pi)})|\leq \delta_r$.
\end{definition}

Sufficiency ensures the evidence is self-contained; Necessity prevents trivial solutions (e.g., selecting the entire bag); Recoverability bridges training and inference.
Together, they separate ``correct prediction'' from ``correct explanation'' as two evaluation axes.
Appendix Table~\ref{tab:metric_defs} summarizes the operational diagnostics used in the experiments.
The next section presents GCE-MIL, which simultaneously addresses \textbf{(P1)}, \textbf{(P2)}, and \textbf{(P3)}.

\section{Grounded Continuous Evidence MIL}
\label{sec:method}

\subsection{Overview}
\label{sec:method_overview}
GCE-MIL is a plug-in wrapper that adds evidence optimization to any existing MIL backbone $f_\theta$ without modifying its architecture.
The wrapper introduces three components, each targeting one of the S/N/R criteria identified in Section~\ref{sec:preliminaries} (Figure~\ref{fig:method_overview}):
\begin{itemize}[leftmargin=1.2em,itemsep=2pt]
\item A \emph{semantic anchor bank} that grounds patch selection in pathology-specific concepts, addressing Necessity~\textbf{(P2)} by tying evidence to diagnostic structures rather than arbitrary attention scores.
\item A \emph{continuous selector with noisy-OR coverage} that produces soft gates $\pi\in[0,1]^N$, addressing Sufficiency~\textbf{(P1)} by modeling multi-source evidence coverage across diagnostic concepts.
\item A \emph{discrete recovery} procedure that converts $\pi$ into an inference-time subset using the same submodular coverage utility, addressing Recoverability~\textbf{(P3)} by keeping the discrete evidence close to the continuous selector.
\end{itemize}
The backbone $f_\theta$ remains structurally unchanged---GCE only adds a soft evidence mask that modulates the backbone's inputs.
The gate $\pi$ is injected according to the backbone's aggregation type: as an attention-logit bias $\alpha_i \leftarrow \alpha_i+\log \pi_i$ for attention-based backbones (ABMIL, CLAM-SB, IBMIL), as feature reweighting $h_i \leftarrow \pi_i\cdot h_i$ for token-based backbones (TransMIL, DTFD-MIL, HDMIL, CAMIL), or as a hybrid of both for multi-path backbones (DSMIL, MHIM-MIL).
This injection preserves the host backbone's scoring head while giving GCE a consistent interface for evidence evaluation.

\begin{figure*}[t]
\centering
\includegraphics[width=\textwidth]{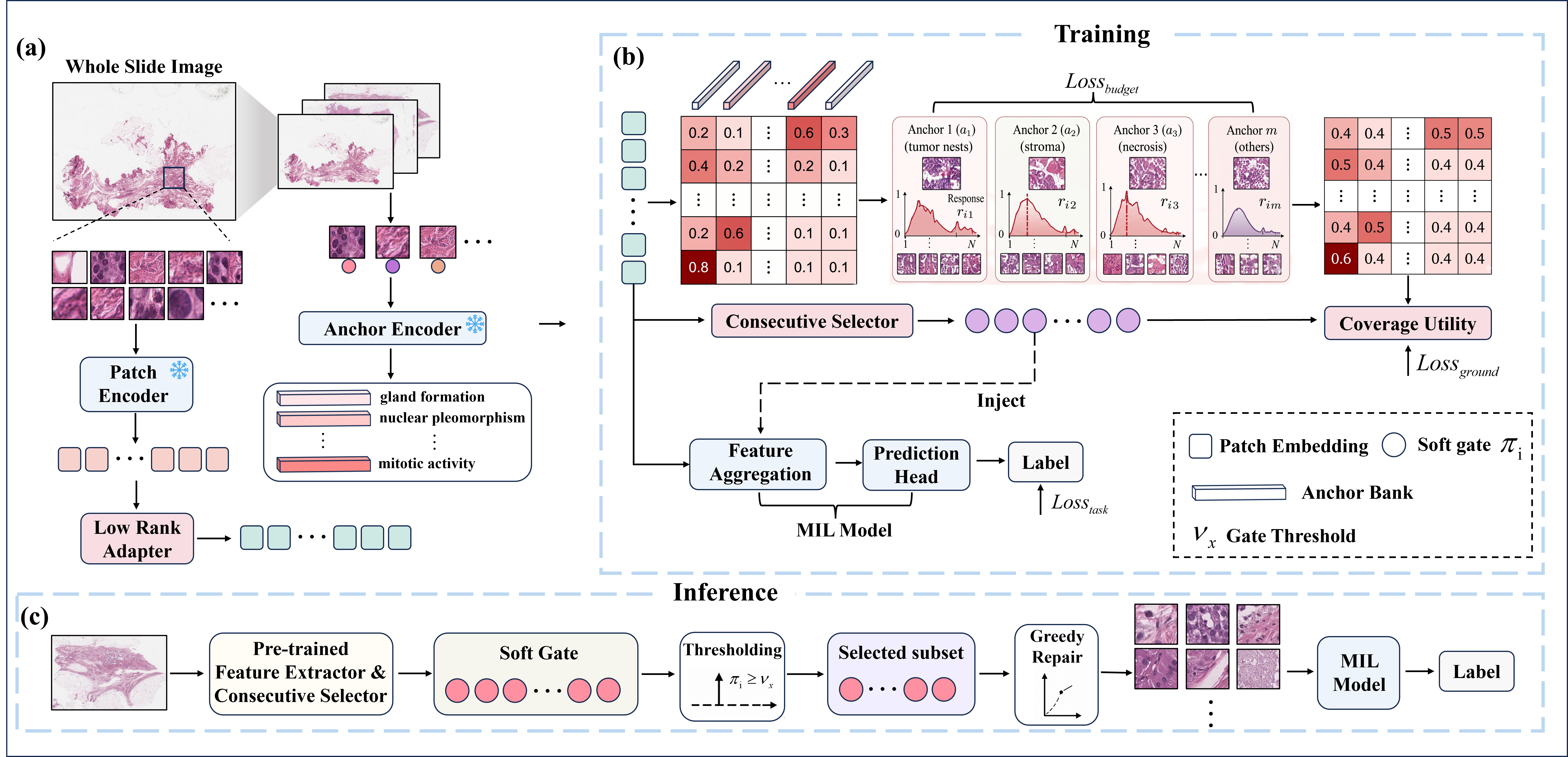}
\caption{\textbf{GCE-MIL architecture.}
The framework wraps existing MIL backbones with three components: (1) low-rank adapter and semantic bridge for anchor grounding (drives Necessity via concept coverage), (2) continuous selector with exact noisy-OR coverage (drives Sufficiency via multi-source evidence), (3) threshold-plus-repair discrete recovery (drives Recoverability via the same marginal coverage objective).
The host backbone remains unchanged, making GCE a plug-in wrapper.}
\label{fig:method_overview}
\end{figure*}

\subsection{Semantic Anchor Grounding}
\label{sec:anchor_bridge}
GCE-MIL grounds evidence selection in domain-specific semantic concepts rather than learning selection from classification gradients alone.
Diagnostically relevant structures---tumor nests, stromal reactions, necrosis, mitotic figures---have well-defined morphological descriptions that can serve as selection anchors, enabling the selector to distinguish diagnostically informative patches from visually salient but irrelevant ones.

GCE-MIL uses $M=8$ semantic anchors defined as frozen text embeddings from TITAN \citep{ding2025multimodal}, a pathology vision-language model.
The anchor prompts are task-specific morphology descriptions chosen before training from disease and histology priors, without validation-set tuning or patch-level concept labels.
Examples include ``gland formation,'' ``nuclear pleomorphism,'' ``mitotic activity,'' and ``necrotic tumor cells''; Appendix Table~\ref{tab:anchor_protocol} records the prompt categories used for each dataset family.
The text embeddings are computed once and fixed during training, and only the adapter, bridge, selector, and host MIL parameters are learned.

Each patch embedding passes through a low-rank residual adapter $e_i=\mathrm{norm}((I+UV^\top)h_i)$, where $U,V\in\mathbb{R}^{d\times r}$ ($r=32$) are initialized near zero.
A separate bridge $B(\cdot)$ maps raw features into the anchor space.
The patch-anchor response is:
\begin{equation}
r_{im}=\sigma\left(\gamma(\cos(B(h_i),a_m)-\delta)\right),
\label{eq:anchor_response}
\end{equation}
where $a_m$ is the frozen anchor embedding, $\gamma=8.0$ sharpens the response, and $\delta=0.15$ suppresses weak matches.
Disease-specific anchors improve over generic prompts in Table~\ref{tab:grounding}, lowering the C-D gap from $0.015$ to $0.010$ and raising complement degradation from $0.210$ to $0.290$.
Random and shuffled prompts remain close to no grounding, while generic prompts, disease-specific prompts, and constrained TITAN grounding improve in order; this pattern suggests that the gain is not only an effect of adding selector capacity.
The full TITAN anchor configuration with the constrained bridge further reaches $0.004$ C-D gap and $0.412$ complement degradation.

\subsection{Continuous Selector and Noisy-OR Coverage}
\label{sec:selector_coverage}
Given the anchor responses $\{r_{im}\}$, the continuous selector determines which patches to include in the evidence subset.
For each patch, a small MLP receives the adapted embedding $e_i$ and spatial coordinates $c_i$, and outputs a scalar score $s_i$.
The inclusion gate is computed as:
\begin{equation}
\pi_i=\sigma\left(\frac{s_i-\nu_x}{T}\right),
\label{eq:selector}
\end{equation}
where $\nu_x=0$ is a centering constant and $T$ is a temperature that is annealed from $1.0$ to $0.4$ during training.
This annealing gradually pushes the gate distribution toward a bimodal regime (Figure~\ref{fig:motivation_main}, middle panel), making the continuous selector increasingly discrete-like and facilitating recovery at inference time.

\textbf{Why noisy-OR for coverage?}
The S/N/R criteria impose specific requirements on how per-patch anchor responses are aggregated into coverage.
Mean pooling conflates ``many weak responses'' with ``one strong response,'' violating coverage semantics.
Attention pooling reintroduces softmax concentration.
Noisy-OR provides the right inductive bias: for anchor $m$, coverage under continuous gates $\pi$ is
\begin{equation}
v_m(\pi)=1-\prod_i(1-\pi_i r_{im}).
\label{eq:noisy_or}
\end{equation}
This models each patch as an independent evidence channel with diminishing marginal returns.
The class-level utility aggregates coverage across anchors:
\begin{equation}
U_c(\pi)=\sum_m \alpha_{cm}v_m(\pi), \qquad \alpha_{cm}\ge 0,
\label{eq:class_utility}
\end{equation}
where $\alpha_{cm}$ are learnable class-anchor weights.
Crucially, noisy-OR provides closed-form marginals for greedy repair: $\partial U_c/\partial \pi_i = \sum_m \alpha_{cm}r_{im}\prod_{j\neq i}(1-\pi_j r_{jm})$.
The marginal gain decreases as more patches are selected, which is the diminishing-returns property needed for Necessity.
Appendix~\ref{app:theory} gives the corresponding modeling interpretation, including S/N/R independence, a gate-margin recoverability bound, conditional coverage bounds, and a Cox risk-pathway view.

\begin{proposition}[Submodularity of Noisy-OR Coverage]
\label{prop:submodular}
For fixed anchor responses $r_{im}$ and class weights $\alpha_{cm} \ge 0$, the utility $U_c(S) = \sum_m \alpha_{cm} [1 - \prod_{i \in S} (1 - r_{im})]$ is monotone submodular in $S$.
\end{proposition}
\begin{proof}
For $S \subseteq T$ and $i \notin T$, the marginal gain is $\Delta_m(i | S) = r_{im} \prod_{j \in S} (1 - r_{jm}) \ge r_{im} \prod_{j \in T} (1 - r_{jm}) = \Delta_m(i | T)$, since $S \subseteq T$ implies the product over $S$ is at least as large.
Summing over $m$ with $\alpha_{cm} \ge 0$ preserves the inequality.
\end{proof}
This submodularity justifies greedy marginal repair at the coverage-utility level: under the standard cardinality-limited coverage setting, greedy selection attains the usual $(1-1/e)$ approximation \citep{nemhauser1978analysis}; Appendix~\ref{app:theory} gives the curvature-aware refinement.
The implemented repair additionally checks threshold recovery and prediction sufficiency, so the claim is a coverage-property statement rather than a global optimality statement about the classifier.

\subsection{Training Objective and Discrete Recovery}
\label{sec:training_recovery}
GCE-MIL trains the host backbone and selector jointly with a composite loss:
\begin{equation}
\mathcal{L}
=\mathcal{L}_{\mathrm{task}}
+\lambda_b\mathcal{L}_{\mathrm{budget}}
+\lambda_g\mathcal{L}_{\mathrm{ground}},
\label{eq:loss}
\end{equation}
where each term targets a specific S/N/R criterion.
$\mathcal{L}_{\mathrm{task}}$ is the unmodified backbone loss (cross-entropy for classification, Cox partial likelihood for survival), preserving the host model's predictive capacity.
$\mathcal{L}_{\mathrm{budget}}=\operatorname{ReLU}(\mathbb{E}[\pi]-\rho)^2$ enforces sparsity, driving \textbf{Sufficiency} by requiring the selector to preserve the prediction with a compact subset.
The reported benchmark uses the operating evidence budget $\rho=0.05$, selected by the validation sweep in Appendix Table~\ref{tab:budget}; larger budgets are reported as sensitivity points rather than mixed into the main tables.
$\mathcal{L}_{\mathrm{ground}}$ aligns $\pi$ with noisy-OR anchor responses, driving \textbf{Necessity} by ensuring selected patches are grounded in diagnostic concepts rather than arbitrary features.
Recoverability is enforced by temperature annealing and the threshold-plus-repair procedure, which make the learned continuous gate compatible with discrete evidence extraction at inference.
The weights $\lambda_b=0.1$ and $\lambda_g=0.5$ define the reported cross-dataset setting and are kept fixed across datasets and backbones after selection on BRACS validation folds.
Table~\ref{tab:ablation} validates each component's contribution: adding budget control reduces the C-D gap from $0.055$ to $0.011$; adding grounding increases complement degradation from $0.318$ to $0.403$; the full pipeline reaches $0.004$ gap and $0.412$ degradation.

\textbf{Discrete recovery at inference.}
At test time, GCE-MIL converts the continuous selector into a discrete evidence subset via threshold-plus-repair (Algorithm~\ref{alg:recovery}).
The initial subset $S_0 = \{i : \pi_i > 0.5\}$ is obtained by thresholding; if empty, the top-$1$ patch is used as a fallback.
Greedy repair then adds patches in decreasing order of marginal coverage gain until the coverage target $c=0.95$ is met.
Because the coverage utility is monotone submodular (Proposition~\ref{prop:submodular}), this greedy procedure has a principled diminishing-returns objective rather than an unrelated post-hoc ranking.
The pseudocode is provided in Appendix~\ref{app:recovery_algorithm}.

\begin{proposition}[Greedy recovery scope]
\label{prop:recovery_scope}
Let $\pi \in [0,1]^N$ be the continuous selector, $S_0 = \{i : \pi_i > 0.5\}$ be the thresholded subset, and $S^*$ be the output of Algorithm~\ref{alg:recovery} with coverage target $c$.
Then the following statements hold:
\begin{enumerate}[leftmargin=1.4em,itemsep=1pt,topsep=2pt]
\item If the loop terminates by satisfying the coverage condition, then $\min_m v_m(\mathbf{1}_{S^*}) \ge c$ by construction.
\item Each added patch maximizes the exact one-step marginal gain of the noisy-OR utility used during training.
\item If repair is restricted to a fixed-size shortlist and evaluated only as coverage maximization, the greedy part inherits the standard $(1-1/e)$ approximation to the best shortlist subset of that size.
\end{enumerate}
These are coverage-level properties; they do not assert global optimality of the host classifier under arbitrary interventions.
\end{proposition}

\begin{proof}[Proof sketch]
The first claim follows directly from the termination condition.
The second follows because Algorithm~\ref{alg:recovery} ranks candidates by $\partial U_c/\partial \pi_i$ computed from the noisy-OR utility.
The third follows from standard greedy analysis for monotone submodular maximization under a cardinality budget \citep{nemhauser1978analysis}; prediction sufficiency is then checked empirically by the intervention diagnostics rather than assumed by the theorem.
\end{proof}

\section{Experiments}
\label{sec:experiments}

\subsection{Setup}
\label{sec:setup}
\textbf{Datasets.}
The evaluation covers $9$ datasets spanning two tasks: $4$ classification benchmarks (BRACS, PANDA, TCGA-BRCA, TCGA-NSCLC) and $5$ survival cohorts (TCGA-LUAD, STAD, UCEC, KIRP, KIRC).
These datasets cover diverse tissue types, label granularities (7-class fine-grained to binary subtyping), and bag sizes (hundreds to tens of thousands of patches).
Dataset details are provided in Appendix~\ref{app:datasets}.

\textbf{Backbones.}
GCE-MIL is attached to $9$ host backbones spanning the major MIL families: attention-based (ABMIL \citep{ilse2018attention}, CLAM-SB \citep{lu2021data}, IBMIL \citep{lin2023interventional}), transformer-based (TransMIL \citep{shao2021transmil}), dual-stream (DSMIL \citep{li2021dual}), pseudo-bag (DTFD-MIL \citep{zhang2022dtfd}), hard-mining (MHIM-MIL \citep{tang2023multiple}), context-aware (CAMIL \citep{fourkioticamil}), and hierarchical (HDMIL \citep{dong2025fast}).
This $9\times 9$ grid ($81$ configurations) tests whether GCE generalizes across backbone architectures and dataset characteristics.
All training and evaluation protocol details are provided in Appendix~\ref{app:implementation}.

\subsection{Main Classification Results}
\label{sec:main_results}
Table~\ref{tab:cls_main} reports the classification half of the $9\times 9$ benchmark; the survival half is reported in Appendix Table~\ref{tab:surv_main}.
The central question is whether optimizing evidence quality changes slide-level prediction, or merely reshuffles which patches are selected without affecting accuracy.
GCE gives positive Macro-F1 changes on most backbone-dataset pairs, with the largest gains on the most challenging dataset (BRACS, 7-class fine-grained classification).
On BRACS, ABMIL improves from $0.634$ to $0.703$ Macro-F1 ($+6.9$ points), HDMIL from $0.707$ to $0.773$ ($+6.6$ points), and IBMIL from $0.776$ to $0.801$ ($+2.5$ points).
The gains are smaller on easier datasets (NSCLC, binary subtyping) where baselines already achieve $>0.90$ Macro-F1, consistent with the expectation that evidence optimization matters most when the classification task requires integrating multiple diagnostic concepts.
On PANDA, GCE preserves performance (HDMIL: $0.696\to0.701$) while compacting evidence to $\sim5\%$ of patches.
On BRCA, the hardest binary task, CLAM-SB improves from $0.523$ to $0.576$ and MHIM-MIL from $0.556$ to $0.600$.

\begin{table*}[t]
\centering
\caption{Classification performance on four histopathology benchmarks (5-fold cross-validation). Baseline rows report absolute mean$\pm$std; \gce{} rows report the change from the immediately preceding baseline using triangle markers.}
\label{tab:cls_main}
\resizebox{\textwidth}{!}{%
\begin{tabular}{l ccc ccc ccc ccc}
\toprule
& \multicolumn{3}{c}{BRACS} & \multicolumn{3}{c}{NSCLC} & \multicolumn{3}{c}{PANDA} & \multicolumn{3}{c}{BRCA} \\
\cmidrule(lr){2-4}\cmidrule(lr){5-7}\cmidrule(lr){8-10}\cmidrule(lr){11-13}
Method & Accuracy & Macro-F1 & AUC & Accuracy & Macro-F1 & AUC & Accuracy & Macro-F1 & AUC & Accuracy & Macro-F1 & AUC \\
\midrule
ABMIL & 0.754$\pm$0.028 & 0.634$\pm$0.045 & 0.864$\pm$0.019 & 0.909$\pm$0.014 & 0.895$\pm$0.016 & 0.948$\pm$0.009 & 0.701$\pm$0.015 & 0.643$\pm$0.018 & 0.925$\pm$0.008 & 0.805$\pm$0.027 & 0.449$\pm$0.049 & 0.803$\pm$0.022 \\
\gce{} $\Delta$ & \uptri{0.036} & \uptri{0.069} & \uptri{0.051} & \uptri{0.023} & \uptri{0.022} & \uptri{0.013} & \uptri{0.007} & \uptri{0.000} & \uptri{0.009} & \uptri{0.014} & \uptri{0.042} & \uptri{0.017} \\
\midrule
CLAM-SB & 0.807$\pm$0.025 & 0.742$\pm$0.038 & 0.919$\pm$0.016 & 0.923$\pm$0.012 & 0.911$\pm$0.014 & 0.956$\pm$0.007 & 0.723$\pm$0.013 & 0.665$\pm$0.014 & 0.931$\pm$0.006 & 0.818$\pm$0.024 & 0.523$\pm$0.042 & 0.821$\pm$0.019 \\
\gce{} $\Delta$ & \uptri{0.015} & \uptri{0.023} & \uptri{0.008} & \uptri{0.017} & \uptri{0.017} & \uptri{0.000} & \uptri{0.006} & \uptri{0.012} & \uptri{0.001} & \uptri{0.021} & \uptri{0.053} & \uptri{0.018} \\
\midrule
TransMIL & 0.738$\pm$0.031 & 0.676$\pm$0.048 & 0.883$\pm$0.021 & 0.936$\pm$0.013 & 0.924$\pm$0.015 & 0.969$\pm$0.008 & 0.730$\pm$0.014 & 0.670$\pm$0.016 & 0.923$\pm$0.007 & 0.820$\pm$0.025 & 0.548$\pm$0.045 & 0.813$\pm$0.021 \\
\gce{} $\Delta$ & \uptri{0.030} & \uptri{0.038} & \uptri{0.021} & \uptri{0.000} & \uptri{0.000} & \uptri{0.003} & \uptri{0.009} & \downtri{0.002} & \uptri{0.012} & \uptri{0.000} & \uptri{0.013} & \uptri{0.015} \\
\midrule
DSMIL & 0.765$\pm$0.029 & 0.684$\pm$0.044 & 0.888$\pm$0.018 & 0.930$\pm$0.012 & 0.915$\pm$0.015 & 0.962$\pm$0.007 & 0.696$\pm$0.016 & 0.628$\pm$0.019 & 0.911$\pm$0.008 & 0.846$\pm$0.022 & 0.529$\pm$0.044 & 0.839$\pm$0.018 \\
\gce{} $\Delta$ & \uptri{0.014} & \uptri{0.042} & \uptri{0.021} & \uptri{0.015} & \uptri{0.015} & \uptri{0.008} & \uptri{0.006} & \uptri{0.010} & \uptri{0.005} & \downtri{0.001} & \uptri{0.014} & \uptri{0.000} \\
\midrule
DTFD-MIL & 0.784$\pm$0.026 & 0.691$\pm$0.041 & 0.917$\pm$0.015 & 0.933$\pm$0.011 & 0.921$\pm$0.013 & 0.965$\pm$0.007 & 0.736$\pm$0.012 & 0.676$\pm$0.015 & 0.932$\pm$0.006 & 0.899$\pm$0.018 & 0.684$\pm$0.032 & 0.840$\pm$0.016 \\
\gce{} $\Delta$ & \uptri{0.026} & \uptri{0.066} & \uptri{0.009} & \uptri{0.007} & \uptri{0.000} & \uptri{0.003} & \uptri{0.002} & \uptri{0.012} & \uptri{0.004} & \uptri{0.000} & \uptri{0.007} & \uptri{0.011} \\
\midrule
IBMIL & 0.828$\pm$0.022 & 0.776$\pm$0.035 & 0.918$\pm$0.014 & 0.935$\pm$0.010 & 0.922$\pm$0.012 & 0.968$\pm$0.006 & 0.715$\pm$0.014 & 0.648$\pm$0.016 & 0.923$\pm$0.007 & 0.814$\pm$0.023 & 0.487$\pm$0.039 & 0.829$\pm$0.017 \\
\gce{} $\Delta$ & \uptri{0.008} & \uptri{0.025} & \uptri{0.013} & \uptri{0.014} & \uptri{0.014} & \downtri{0.003} & \uptri{0.019} & \uptri{0.034} & \uptri{0.007} & \uptri{0.020} & \uptri{0.060} & \uptri{0.007} \\
\midrule
MHIM-MIL & 0.796$\pm$0.025 & 0.718$\pm$0.042 & 0.899$\pm$0.017 & 0.938$\pm$0.011 & 0.926$\pm$0.013 & 0.970$\pm$0.007 & 0.713$\pm$0.013 & 0.655$\pm$0.015 & 0.919$\pm$0.007 & 0.839$\pm$0.021 & 0.556$\pm$0.038 & 0.846$\pm$0.016 \\
\gce{} $\Delta$ & \uptri{0.012} & \uptri{0.030} & \uptri{0.023} & \uptri{0.007} & \uptri{0.006} & \uptri{0.000} & \uptri{0.012} & \uptri{0.017} & \uptri{0.004} & \uptri{0.017} & \uptri{0.044} & \uptri{0.012} \\
\midrule
CAMIL & 0.738$\pm$0.030 & 0.666$\pm$0.049 & 0.879$\pm$0.020 & 0.944$\pm$0.011 & 0.931$\pm$0.013 & 0.975$\pm$0.006 & 0.719$\pm$0.014 & 0.657$\pm$0.017 & 0.926$\pm$0.008 & 0.839$\pm$0.022 & 0.608$\pm$0.043 & 0.811$\pm$0.018 \\
\gce{} $\Delta$ & \uptri{0.039} & \uptri{0.054} & \uptri{0.029} & \uptri{0.006} & \uptri{0.000} & \downtri{0.002} & \uptri{0.006} & \uptri{0.000} & \uptri{0.004} & \uptri{0.003} & \uptri{0.002} & \uptri{0.006} \\
\midrule
HDMIL & 0.781$\pm$0.027 & 0.707$\pm$0.040 & 0.891$\pm$0.018 & 0.940$\pm$0.010 & 0.927$\pm$0.012 & 0.970$\pm$0.006 & 0.746$\pm$0.011 & 0.696$\pm$0.013 & 0.932$\pm$0.006 & 0.866$\pm$0.019 & 0.603$\pm$0.035 & 0.835$\pm$0.017 \\
\gce{} $\Delta$ & \uptri{0.038} & \uptri{0.066} & \uptri{0.038} & \uptri{0.000} & \uptri{0.006} & \uptri{0.003} & \uptri{0.000} & \uptri{0.005} & \uptri{0.006} & \uptri{0.006} & \uptri{0.039} & \uptri{0.009} \\
\bottomrule
\end{tabular}}
\end{table*}

\subsection{Ablation Study}
\label{sec:ablations}
Table~\ref{tab:ablation} isolates the contribution of each component.
A naive selector leaves a large C-D gap ($0.055$) and weak complement degradation ($0.090$); adding budget control, recovery, and grounding progressively yields the full GCE result: $0.748$ Macro-F1, $0.004$ C-D gap, and $0.412$ complement degradation.

\begin{table}[H]
\centering
\caption{Component ablation on BRACS. Each row adds one module to the pipeline.}
\label{tab:ablation}
\scriptsize
\setlength{\tabcolsep}{2pt}
\renewcommand{\arraystretch}{0.86}
\resizebox{\columnwidth}{!}{%
\begin{tabular}{lcccc}
\toprule
Variant & Macro-F1 & C-D Gap$\downarrow$ & Compl.~Degr.$\uparrow$ & Evid.~Suff.$\uparrow$ \\
\midrule
Backbone only        & 0.699$\pm$0.047 & ---   & ---  & --- \\
Naive selector       & 0.708$\pm$0.045 & 0.055$\pm$0.014  & 0.090$\pm$0.023 & 0.381$\pm$0.061 \\
+\,Budget control    & 0.738$\pm$0.036 & 0.011$\pm$0.004  & 0.318$\pm$0.058 & 0.596$\pm$0.094 \\
+\,Discrete recovery & 0.744$\pm$0.033 & 0.006$\pm$0.002  & 0.377$\pm$0.071 & 0.630$\pm$0.102 \\
+\,Semantic grounding  & 0.746$\pm$0.035 & 0.005$\pm$0.002  & 0.403$\pm$0.082 & 0.649$\pm$0.111 \\
Full GCE             & \best{0.748}$\pm$0.032 & \best{0.004}$\pm$0.001 & \best{0.412}$\pm$0.086 & \best{0.659}$\pm$0.108 \\
\bottomrule
\end{tabular}}
\end{table}

The ablation also separates the three S/N/R mechanisms.
Budget control gives the first large improvement: the gap falls from $0.055$ to $0.011$ and complement degradation rises from $0.090$ to $0.318$, consistent with the claim that sparse evidence must carry more of the decision.
Discrete recovery contributes mainly to Recoverability ($0.011\to0.006$), while semantic grounding contributes mainly to Necessity ($0.377\to0.403$), showing that anchor coverage is not merely a parameter increase but changes which patches are treated as decision-critical.

\subsection{Qualitative Evidence Behavior}
\label{sec:qualitative_main}
Figure~\ref{fig:attention_zero} places the learned evidence mask next to the host attention map.
The qualitative pattern matches the intervention diagnostics: attention highlights broad high-score regions, while GCE recovers a compact subset that remains spatially coherent after thresholding and repair.
This figure is included in the main text because it clarifies what the S/N/R metrics measure at the slide level.

\begin{figure*}[t]
\centering
\includegraphics[width=0.90\textwidth]{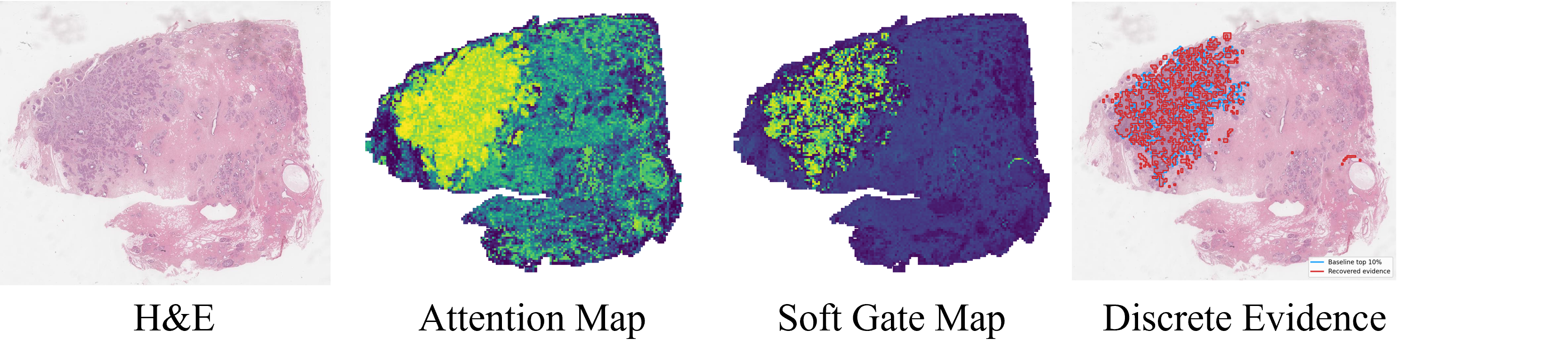}
\caption{\textbf{Main qualitative evidence example.} The host attention map and the recovered GCE evidence subset are shown on the same slide. GCE selects a compact, recoverable evidence set rather than simply visualizing the original attention ranking.}
\label{fig:attention_zero}
\end{figure*}

\subsection{Intervention and Budget-Matched Evidence}
\label{sec:causal_interventions}
\label{sec:same_budget_main}
Table~\ref{tab:causal} reports the direct intervention diagnostic for Sufficiency and Necessity.
Keeping attention top-$k$ changes the full-bag score by \downtri{0.078}, whereas keeping GCE evidence changes it by only \uptri{0.004}; removing attention changes the score by \downtri{0.033}, but removing GCE evidence causes a \downtri{0.176} drop.
Table~\ref{tab:same_budget_main} controls for evidence size by forcing all subset rules to select approximately $5\%$ of each BRACS bag.
At the same budget, attention top-$k$ reaches $0.597$ Macro-F1 and $0.151$ complement degradation, whereas discrete GCE reaches $0.748$ and $0.412$; the prediction gap falls from $0.029$ to $0.004$.
The gains therefore do not come from selecting more tissue, but from recovering a subset that is sufficient, necessary, and discrete-faithful.

\begin{table*}[t]
\centering
\begin{minipage}[t]{0.48\textwidth}
\centering
\caption{\textbf{Intervention diagnostic.} Values are changes relative to the full-bag score averaged over the nine-dataset benchmark.}
\label{tab:causal}
\resizebox{\linewidth}{!}{%
\begin{tabular}{lcc}
\toprule
Subset rule & Keep only $\Delta$ & Remove $\Delta$ \\
\midrule
Random-$k$        & \downtri{0.180} & \uptri{0.001} \\
Attention top-$k$ & \downtri{0.078} & \downtri{0.033} \\
Saliency top-$k$  & \downtri{0.063} & \downtri{0.042} \\
\midrule
GCE evidence      & \best{\uptri{0.004}} & \best{\downtri{0.176}} \\
\bottomrule
\end{tabular}}
\end{minipage}
\hfill
\begin{minipage}[t]{0.48\textwidth}
\centering
\caption{\textbf{Same-budget evidence comparison on BRACS.} All subset rules select approximately $5\%$ of patches.}
\label{tab:same_budget_main}
\resizebox{\linewidth}{!}{%
\begin{tabular}{lccc}
\toprule
Subset rule & Macro-F1$\uparrow$ & Gap$\downarrow$ & Compl.~Degr.$\uparrow$ \\
\midrule
Random-$k$        & 0.431$\pm$0.054 & 0.043$\pm$0.012 & 0.048$\pm$0.015 \\
Attention top-$k$ & 0.597$\pm$0.045 & 0.029$\pm$0.008 & 0.151$\pm$0.036 \\
Gradient top-$k$  & 0.613$\pm$0.041 & 0.025$\pm$0.006 & 0.166$\pm$0.031 \\
Occlusion top-$k$ & 0.628$\pm$0.038 & 0.022$\pm$0.005 & 0.196$\pm$0.039 \\
\midrule
GCE discrete      & \best{0.748}$\pm$0.032 & \best{0.004}$\pm$0.001 & \best{0.412}$\pm$0.086 \\
\bottomrule
\end{tabular}}
\end{minipage}
\end{table*}

\section{Conclusion}
\label{sec:conclusion}
This work formalizes the gap between classification accuracy and evidence quality in MIL through three criteria---Sufficiency, Necessity, and Recoverability---and shows that existing attention-based methods fail all three.
GCE-MIL addresses these failures with semantic anchor grounding, noisy-OR coverage with closed-form marginals, and threshold-plus-repair discrete recovery.
Across $81$ backbone-dataset configurations, GCE-MIL improves both prediction and evidence quality, and optional tile prefiltering enables up to $5\times$ faster end-to-end inference at $0.989\times$ relative utility.

\clearpage
\nocite{brancati2022bracs,bulten2022artificial,weinstein2013cancer,ehteshami2017diagnostic,bandi2018detection,campanella2019clinical,pantanowitz2011review,verghese2023computational,song2023artificial,dietterich1997solving,maron1997framework,wang2019weakly,chen2021annotation,li2019attention,xu2014weakly,kraus2016classifying,tourniaire2023ms,yao2020whole,chen2021whole,cheng2021computational,kang2023benchmarking,chen2024towards,lu2023visual,kapse2025gecko,he2016deep,russakovsky2015imagenet,oquab2024dinov2,he2020momentum,chen2020simple,caron2021emerging,van2008visualizing,miller1969froc,bunch1978free,kingma2014adam,srivastava2014dropout,hinton2015distilling,tarvainen2017mean,martins2016softmax,tsallis1988possible,cover1999elements,shannon1948mathematical,qu2024rethinking,shao2025multiple,guo2023higt,buzzard2024paths,dong2025fast,yang2024mambamil,du2025rethinking,xiong2021nystromformer}
\bibliographystyle{plainnat}
\bibliography{references_qill0506}

\clearpage
\appendix
\makeatletter
\renewcommand{\section}{%
  \@startsection{section}{1}{\z@}%
                {-1.1ex \@plus -0.2ex \@minus -0.1ex}%
                {0.6ex \@plus 0.1ex}%
                {\large\bf\raggedright}%
}
\renewcommand{\subsection}{%
  \@startsection{subsection}{2}{\z@}%
                {-0.9ex \@plus -0.2ex \@minus -0.1ex}%
                {0.35ex \@plus 0.1ex}%
                {\normalsize\bf\raggedright}%
}
\makeatother
\setlength{\parskip}{0pt}
\setlength{\textfloatsep}{6pt plus 1pt minus 1pt}
\setlength{\floatsep}{6pt plus 1pt minus 1pt}
\setlength{\intextsep}{6pt plus 1pt minus 1pt}

\section{Overview}
\label{app:overview}
This appendix provides the theoretical, experimental, and implementation details that complement the main paper.
We first state the scope of the theoretical claims and derive the coverage-level properties of noisy-OR recovery in Appendix~\ref{app:theory}.
We then give the detailed discrete recovery algorithm in Appendix~\ref{app:recovery_algorithm}, followed by additional evidence diagnostics and baseline comparisons in Appendix~\ref{app:evidence_diagnostics}.
Appendix~\ref{app:survival_results} reports the survival prediction results and explains how the same coverage view extends to Cox risk modeling.
Appendix~\ref{app:additional_experiments} presents CAMELYON-16 localization and significance tests, while Appendix~\ref{app:ablation_details} expands the ablation and sensitivity analyses.
Appendix~\ref{app:backbone_generalization} reports multi-backbone and multi-encoder generalization, and Appendix~\ref{app:failure} summarizes the failure-case audit.
Finally, Appendix~\ref{app:implementation} provides hardware, training, anchor-prompt, and computational-cost details; Appendix~\ref{app:datasets} describes datasets and preprocessing; Appendix~\ref{app:figures} collects additional visualizations; and Appendix~\ref{app:limitations} discusses limitations and broader implications.

\section{Theoretical Foundation and Notes}
\label{app:theory}

\noindent\fbox{%
\begin{minipage}{0.96\linewidth}
\textbf{Scope of theory.}
The results in this appendix are coverage-level statements about the noisy-OR utility, not claims of global classifier optimality.
Submodularity justifies the marginal-gain recovery objective under fixed anchor responses, while the conditional interventional bound explains when uncovered anchor mass can upper-bound omitted representation contribution.
Classifier-level faithfulness is therefore evaluated empirically through keep/remove interventions and, separately, through CAMELYON-16 localization against pixel-level annotations.
\end{minipage}}
\vspace{0.6em}

\subsection{Coverage as an Interventional Proxy}
Noisy-OR is used as a differentiable proxy for intervention search rather than as the definition of evidence faithfulness.
The target quantity is the intervention variation $V(S)=\|f(X)-f(X_S)\|$, which is expensive to optimize directly over discrete subsets.
The following proposition states the assumption under which uncovered anchor mass controls this variation.

\begin{proposition}[Coverage residual bounds representation variation]
\label{prop:interventional_bound}
Let a MIL predictor decompose as $f(X)=q(g(X))$, where $q$ is $L_q$-Lipschitz and the bag representation is additive, $g(X)=\sum_i w_i h_i$ with $w_i\ge 0$.
Assume the anchor bank is $\eta$-complete for the omitted patch contribution: for any subset $S$,
\[
\left\|\sum_{i\notin S} w_i h_i\right\|
\leq \eta \sum_m \alpha_m \prod_{i\in S}(1-r_{im}),
\qquad \alpha_m\ge 0 .
\]
Then
\[
\|f(X)-f(X_S)\|
\leq L_q\eta \sum_m \alpha_m\bigl(1-v_m(\mathbf{1}_S)\bigr).
\]
\end{proposition}

\begin{proof}
By Lipschitz continuity,
$\|f(X)-f(X_S)\|\leq L_q\|g(X)-g(X_S)\|$.
For the additive representation, $g(X)-g(X_S)=\sum_{i\notin S}w_i h_i$.
Applying the anchor-completeness assumption gives the stated bound, and $1-v_m(\mathbf{1}_S)=\prod_{i\in S}(1-r_{im})$ by the noisy-OR definition.
\end{proof}

Proposition~\ref{prop:interventional_bound} is intentionally conditional.
It does not claim that anchors perfectly explain every backbone; instead, it states what the noisy-OR objective is approximating when anchors cover the residual directions relevant to the host predictor.
Direct keep-only and remove interventions are therefore still reported: the bound motivates the proxy, while the interventions test whether the proxy actually tracks prediction changes.

\subsection{Independence of S/N/R Criteria}
Sufficiency, Necessity, and Recoverability are separate desiderata rather than three names for the same intervention score.
The following proposition formalizes why all three diagnostics are reported.

\begin{proposition}[S/N/R are logically independent]
\label{prop:snr_independence}
For a fixed model class, none of Sufficiency, Necessity, and Recoverability implies either of the other two in general.
\end{proposition}

\begin{proof}[Proof sketch]
A selector that returns the whole bag is sufficient but not necessary, because its complement is empty or uninformative.
A selector that returns one of several redundant diagnostic regions can be sufficient but not necessary, because the complement still contains another sufficient region.
A selector can be necessary but not sufficient when the model relies on an interaction between the selected region and contextual tissue outside the subset.
Finally, a continuous gate can be recoverable under thresholding while the resulting discrete subset is neither sufficient nor necessary, for example when the gate is stable but selects background or redundant tissue.
Conversely, sufficient or necessary subsets can be generated by gates with many values near the threshold, making them poorly recoverable even when the intervention property holds.
These constructions are model-relative and do not rely on a particular backbone.
\end{proof}

\subsection{Recoverability Bound for Thresholded Gates}
Recoverability is evaluated by comparing the continuous gated input with its thresholded subset.
Let $F_X(\pi)$ denote the prediction of a fixed trained MIL model on slide $X$ after applying gate vector $\pi$ to the bag, and let $S_\tau(\pi)=\{i:\pi_i\ge\tau\}$.

\begin{proposition}[Gate-margin recoverability bound]
\label{prop:recoverability_bound}
If $F_X$ is $L_X$-Lipschitz in the gate vector under norm $\|\cdot\|$, then
\[
\|F_X(\pi)-F_X(\mathbf{1}_{S_\tau(\pi)})\|
\leq
L_X\|\pi-\mathbf{1}_{S_\tau(\pi)}\|.
\]
In particular, if every gate satisfies $\min_i |\pi_i-\tau|\ge m$ and $\pi_i\in[0,1]$, the bound tightens as annealing drives gates toward $\{0,1\}$ and away from the threshold.
\end{proposition}

\begin{proof}
The inequality follows directly from Lipschitz continuity of $F_X$ with respect to the gate vector.
Temperature annealing affects the right-hand side by reducing the distance between the continuous gate and the thresholded binary gate.
The proposition therefore explains the role of annealing in reducing the C-D gap, but it does not assert that the recovered subset is clinically causal evidence.
\end{proof}

\subsection{Why Noisy-OR}
Noisy-OR is a natural coverage model when multiple independent evidence sources can activate the same latent concept.
For WSI evidence, this matches the setting where several patches may each support gland formation, necrosis, or inflammatory response.
Unlike mean pooling, noisy-OR saturates when a concept is already covered; unlike attention, it does not force concepts to compete through a simplex.
This saturation behavior is important for Necessity: once an anchor is covered, adding another redundant patch has lower marginal gain, so greedy recovery is biased toward complementary evidence.

For a monotone submodular utility $U$, the total curvature is
\[
\kappa
=
1-\min_{i:U(\{i\})>0}
\frac{U(V)-U(V\setminus\{i\})}{U(\{i\})},
\qquad 0\le\kappa\le 1 .
\]
For noisy-OR coverage, this curvature is determined by anchor-response redundancy: patches whose anchor responses are already covered by other selected patches have lower late-stage marginal gain.
Under the same cardinality-limited coverage setting used for greedy repair, the standard greedy guarantee can therefore be sharpened to the curvature-aware form
\[
U(S_{\mathrm{greedy}})
\ge
\frac{1-e^{-\kappa}}{\kappa} U(S^\star),
\]
with the convention that the factor equals $1$ when $\kappa=0$.
This does not change the coverage-level scope of the theorem, but it clarifies why greedy repair is most effective when the selected patches cover complementary anchors rather than redundant copies of the same morphology.

\subsection{Risk-Pathway Interpretation for Survival}
For survival tasks, GCE-MIL uses the same selector and coverage utility but trains the slide output as a scalar Cox risk score.
The noisy-OR term has a compatible interpretation if each anchor is viewed as a latent risk pathway.
Let $z_m(X)$ denote whether pathway $m$ is present somewhere in the slide, and let $r_{im}$ be the probability that patch $i$ expresses that pathway.
Under conditional independence of pathway evidence across selected patches,
\[
P(z_m=1\mid S)=1-\prod_{i\in S}(1-r_{im})=v_m(\mathbf{1}_S).
\]
A Cox-style log-risk model can then be written as
\[
h(X_S)=\sum_m \beta_m v_m(\mathbf{1}_S),\qquad
\lambda(t\mid X_S)=\lambda_0(t)\exp(h(X_S)).
\]
This mapping does not turn the method into a full competing-risks model; it only explains why the same saturating coverage form is reasonable for scalar risk prediction.
Once a high-risk morphology is already covered, additional patches showing the same morphology have diminishing marginal effect on the log-risk, whereas a patch activating a different risk pathway can still change the risk score.

\subsection{Sufficiency-Necessity Duality for MIL Classifiers}
Sufficiency and Necessity are complementary but not equivalent.
A subset can preserve the prediction when kept while still being redundant with other evidence in the complement.
Conversely, a subset can be necessary when removed but insufficient on its own if it interacts with context.
GCE-MIL therefore reports both keep-only and remove interventions.

\subsection{Recoverability under Temperature Annealing}
Recoverability depends on the gap between the continuous selector and the thresholded subset.
Temperature annealing encourages $\pi_i$ values to move toward $0$ or $1$, reducing ambiguity near the threshold.
Greedy repair then corrects residual coverage failures after thresholding.
The empirical counterpart is the C-D gap: selector and grounding ablations that leave many gates near the threshold produce larger gaps, while the final model reduces the gap to $0.004$--$0.005$ in the main diagnostics.

\FloatBarrier
\section{Detailed Recovery Algorithm}
\label{app:recovery_algorithm}
Algorithm~\ref{alg:recovery} gives the discrete recovery procedure used at inference.
It first thresholds the continuous gate $\pi$ and then greedily repairs coverage failures using the exact noisy-OR marginal.
The main method section therefore remains focused on the S/N/R design while the full implementation detail is provided here.
The key implementation choice is that repair ranks candidates by the exact marginal of the same noisy-OR utility used during training.
Thus the discrete subset is not a separate heuristic explanation; it is the hard recovery of the trained continuous evidence objective.

\begin{algorithm}[H]
\caption{GCE-MIL Discrete Recovery}
\label{alg:recovery}
\KwIn{continuous gates $\pi$, anchor responses $r$, coverage target $c=0.95$}
\KwOut{discrete subset $S^*$}
$S \leftarrow \{i:\pi_i>0.5\}$\;
\If{$S=\emptyset$}{
  $S \leftarrow \{\arg\max_i \pi_i\}$\;
}
\While{$\min_m v_m(\mathbf{1}_S)<c$}{
  $i^\star \leftarrow \arg\max_{i\notin S}\partial U_c(\mathbf{1}_S)/\partial \pi_i$\;
  $S \leftarrow S\cup\{i^\star\}$\;
}
\Return{$S^*=S$}\;
\end{algorithm}

\section{Evidence Diagnostics and Baseline Comparisons}
\label{app:evidence_diagnostics}
Table~\ref{tab:metric_defs} defines the operational diagnostics used throughout the evidence experiments.
Let $\Phi(\cdot)$ denote the intervention output compared within a slide (class probability for classification, normalized risk for survival), and let $M(\cdot)$ denote the aggregate task metric (Macro-F1 or C-index) evaluated after replacing the full bag by the specified subset.

\begin{table}[!htbp]
\centering
\caption{\textbf{Operational S/N/R diagnostics.} The C-D gap is an optimization-aligned recoverability diagnostic; keep/remove interventions and CAMELYON localization provide additional model-relative and spatial checks.}
\label{tab:metric_defs}
\small
\begin{tabular}{p{0.24\linewidth}p{0.43\linewidth}p{0.22\linewidth}}
\toprule
Diagnostic & Definition & Evidence role \\
\midrule
Keep-only drop & $M(X)-M(X_S)$ after retaining only $S$ & Sufficiency; lower is better \\
Evidence sufficiency & $M(X_S)$, the task metric of the evidence-only bag & Sufficiency; higher is better \\
Complement degradation & $M(X)-M(X_{\neg S})$ after removing $S$ & Necessity; higher is better \\
C-D gap & $\mathbb{E}_{X}\!\left[|\Phi(X_{\pi})-\Phi(X_{S(\pi)})|\right]$ & Recoverability; lower is better \\
\bottomrule
\end{tabular}
\end{table}

Table~\ref{tab:diagnostics_bracs} provides a compact BRACS diagnostic for the S/N/R metrics.
This appendix groups the evidence diagnostics with budget-matched and post-hoc comparisons, so each S/N/R claim is accompanied by the table used to support it.
On BRACS, DTFD-MIL+GCE reaches $0.484$ complement degradation and HDMIL+GCE has the highest evidence sufficiency ($0.685$), while the C-D gap remains below $0.012$ for every host backbone.
This compact table gives single-dataset intuition before the larger three-dataset diagnostic table: different backbones trade off Necessity and Sufficiency, but all recovered subsets remain close to their continuous selectors.

\begin{table}[!htbp]
\centering
\caption{\textbf{BRACS evidence diagnostics.} Single-column S/N/R summary for GCE-wrapped backbones.}
\label{tab:diagnostics_bracs}
\resizebox{\columnwidth}{!}{%
\begin{tabular}{lccc}
\toprule
Method & C-D Gap$\downarrow$ & Compl.~Degr.$\uparrow$ & Evid.~Suff.$\uparrow$ \\
\midrule
ABMIL\gce      & 0.011$\pm$0.004 & 0.230$\pm$0.042 & 0.430$\pm$0.068 \\
CLAM-SB\gce    & 0.006$\pm$0.003 & 0.332$\pm$0.051 & 0.564$\pm$0.077 \\
TransMIL\gce   & 0.006$\pm$0.002 & 0.221$\pm$0.047 & 0.483$\pm$0.084 \\
DSMIL\gce      & 0.008$\pm$0.003 & 0.243$\pm$0.046 & 0.468$\pm$0.082 \\
DTFD-MIL\gce   & 0.005$\pm$0.002 & \best{0.484}$\pm$0.082 & 0.645$\pm$0.091 \\
IBMIL\gce      & 0.005$\pm$0.002 & \second{0.450}$\pm$0.076 & \second{0.660}$\pm$0.088 \\
MHIM-MIL\gce   & \best{0.004}$\pm$0.001 & 0.236$\pm$0.044 & 0.475$\pm$0.081 \\
CAMIL\gce      & 0.006$\pm$0.002 & 0.346$\pm$0.065 & 0.570$\pm$0.095 \\
HDMIL\gce      & \best{0.004}$\pm$0.002 & 0.441$\pm$0.071 & \best{0.685}$\pm$0.094 \\
\bottomrule
\end{tabular}}
\end{table}

Table~\ref{tab:diagnostics} reports the three S/N/R metrics for each backbone after GCE training and discrete recovery.
The C-D gap is an optimization-aligned recoverability diagnostic: HDMIL achieves $0.004$ on BRACS and $0.002$ on NSCLC, indicating that the continuous selector and recovered discrete subset remain close under the reported recovery protocol.
Complement degradation (Necessity) reaches $0.484$ for DTFD-MIL on BRACS, meaning removing the evidence subset drops performance by nearly half---the selected patches are genuinely informative.
Evidence sufficiency reaches $0.685$ for HDMIL on BRACS, indicating that the evidence subset alone retains most of the full-bag prediction.
The pattern is consistent across datasets: backbones with stronger attention mechanisms (DTFD-MIL, IBMIL, HDMIL) achieve higher Necessity and Sufficiency scores, suggesting that GCE amplifies the backbone's existing ability to identify informative patches.
The important point is not that one host backbone dominates every metric.
Instead, every host receives an explicit evidence interface, allowing S/N/R diagnostics to be compared across architectures that otherwise expose different attention or instance scores.

\begin{table}[H]
\centering
\caption{GCE evidence diagnostics per backbone. C-D Gap: continuous--discrete prediction gap (Recoverability); Complement Degradation: performance drop when evidence is removed (Necessity); Evidence Sufficiency: performance retained by evidence alone (Sufficiency).}
\label{tab:diagnostics}
\resizebox{\textwidth}{!}{%
\begin{tabular}{l ccc ccc ccc}
\toprule
& \multicolumn{3}{c}{BRACS} & \multicolumn{3}{c}{NSCLC} & \multicolumn{3}{c}{PANDA} \\
\cmidrule(lr){2-4}\cmidrule(lr){5-7}\cmidrule(lr){8-10}
Method & C-D Gap & Compl.~Degr. & Evid.~Suff. & C-D Gap & Compl.~Degr. & Evid.~Suff. & C-D Gap & Compl.~Degr. & Evid.~Suff. \\
\midrule
ABMIL & 0.011$\pm$0.004 & 0.230$\pm$0.042 & 0.430$\pm$0.068 & 0.008$\pm$0.003 & 0.280$\pm$0.054 & 0.500$\pm$0.075 & 0.012$\pm$0.004 & 0.153$\pm$0.038 & 0.338$\pm$0.059 \\
CLAM-SB & 0.006$\pm$0.003 & 0.332$\pm$0.051 & 0.564$\pm$0.077 & 0.004$\pm$0.001 & 0.313$\pm$0.058 & 0.550$\pm$0.081 & 0.006$\pm$0.002 & 0.263$\pm$0.044 & 0.496$\pm$0.071 \\
TransMIL & 0.006$\pm$0.002 & 0.221$\pm$0.047 & 0.483$\pm$0.084 & 0.002$\pm$0.001 & 0.271$\pm$0.050 & 0.537$\pm$0.092 & 0.006$\pm$0.002 & 0.144$\pm$0.033 & 0.370$\pm$0.065 \\
DSMIL & 0.008$\pm$0.003 & 0.243$\pm$0.046 & 0.468$\pm$0.082 & 0.004$\pm$0.001 & 0.310$\pm$0.049 & 0.559$\pm$0.088 & 0.008$\pm$0.003 & 0.150$\pm$0.031 & 0.370$\pm$0.058 \\
DTFD-MIL & 0.005$\pm$0.002 & \best{0.484}$\pm$0.082 & 0.645$\pm$0.091 & 0.003$\pm$0.001 & \best{0.440}$\pm$0.081 & 0.612$\pm$0.097 & 0.005$\pm$0.001 & \best{0.437}$\pm$0.074 & \second{0.603}$\pm$0.086 \\
IBMIL & 0.005$\pm$0.002 & \second{0.450}$\pm$0.076 & \second{0.660}$\pm$0.088 & 0.003$\pm$0.001 & \second{0.431}$\pm$0.075 & \second{0.651}$\pm$0.102 & 0.005$\pm$0.002 & 0.362$\pm$0.058 & 0.575$\pm$0.079 \\
MHIM-MIL & 0.004$\pm$0.001 & 0.236$\pm$0.044 & 0.475$\pm$0.081 & 0.002$\pm$0.001 & 0.273$\pm$0.052 & 0.517$\pm$0.086 & 0.004$\pm$0.001 & 0.180$\pm$0.035 & 0.422$\pm$0.066 \\
CAMIL & 0.006$\pm$0.002 & 0.346$\pm$0.065 & 0.570$\pm$0.095 & 0.002$\pm$0.001 & 0.408$\pm$0.078 & \second{0.651}$\pm$0.114 & 0.005$\pm$0.002 & 0.285$\pm$0.051 & 0.507$\pm$0.082 \\
HDMIL & 0.004$\pm$0.002 & 0.441$\pm$0.071 & \best{0.685}$\pm$0.094 & 0.002$\pm$0.001 & 0.418$\pm$0.074 & \best{0.666}$\pm$0.110 & 0.003$\pm$0.001 & \second{0.368}$\pm$0.062 & \best{0.611}$\pm$0.095 \\
\bottomrule
\end{tabular}}
\end{table}

\subsection{Same-Budget Comparison}
\label{app:same_budget}
Table~\ref{tab:same_budget} compares selection methods under a fixed $5\%$ budget on BRACS.
Attention top-$k$ achieves $0.597$ Macro-F1 and $0.151$ complement degradation; discrete GCE reaches $0.748$ and $0.412$.
The prediction gap drops from $0.029$ to $0.004$.
Under a fixed budget, GCE improves all three S/N/R criteria rather than simply selecting more patches.
This is the fairest control for sparse evidence methods because it removes subset size as a confounder.
The remaining gap is therefore attributable to how the subset is selected and recovered, not to using more tissue.

\begin{table}[!htbp]
\centering
\caption{Same-budget comparison on BRACS. All methods select approximately 5\% of each bag, except the soft GCE row which uses the continuous selector before discretization.}
\label{tab:same_budget}
\resizebox{\columnwidth}{!}{%
\begin{tabular}{lcccc}
\toprule
Subset Rule & Macro-F1 & Prediction Gap$\downarrow$ & Compl.~Degr.$\uparrow$ & Evid.~Suff.$\uparrow$ \\
\midrule
Random-$k$          & 0.431$\pm$0.054 & 0.043$\pm$0.012 & 0.048$\pm$0.015 & 0.295$\pm$0.047 \\
Attention top-$k$   & 0.597$\pm$0.045 & 0.029$\pm$0.008 & 0.151$\pm$0.036 & 0.434$\pm$0.069 \\
Gradient top-$k$    & 0.613$\pm$0.041 & 0.025$\pm$0.006 & 0.166$\pm$0.031 & 0.454$\pm$0.074 \\
Occlusion top-$k$   & 0.628$\pm$0.038 & 0.022$\pm$0.005 & 0.196$\pm$0.039 & 0.478$\pm$0.081 \\
\midrule
GCE soft evidence     & \best{0.754}$\pm$0.029 & \best{0.003}$\pm$0.001 & \second{0.405}$\pm$0.078 & \best{0.670}$\pm$0.095 \\
GCE discrete evidence & \second{0.748}$\pm$0.032 & \second{0.004}$\pm$0.001 & \best{0.412}$\pm$0.086 & \second{0.659}$\pm$0.108 \\
\bottomrule
\end{tabular}}
\end{table}

\subsection{Stability and Post-hoc Comparison}
\label{app:stability_posthoc}
GCE evidence is more stable than attention under stochastic perturbation: Jaccard overlap between evidence sets across random seeds rises from $0.329$ (attention) to $0.606$ (GCE), and prediction flips fall from $0.128$ to $0.045$ (Tables~\ref{tab:stability} and \ref{tab:explain}).
Compared with post-hoc attribution methods at matched subset sizes, GCE reaches $0.722$ keep-only Macro-F1 and $0.176$ remove drop, vs.\ $0.662$/$0.048$ for integrated gradients and $0.668$/$0.052$ for occlusion.
The joint training of selection and prediction in GCE produces evidence that is both more faithful and more stable than post-hoc methods.

\begin{table}[!htbp]
\centering
\caption{Evidence stability under stochastic perturbation. Jaccard measures overlap between evidence sets across runs.}
\label{tab:stability}
\resizebox{\columnwidth}{!}{%
\begin{tabular}{lccc}
\toprule
Evidence Rule & Jaccard$\uparrow$ & Prediction Flip$\downarrow$ & C-D Gap$\downarrow$ \\
\midrule
Attention top-$k$       & 0.329$\pm$0.043 & 0.128$\pm$0.015 & 0.029$\pm$0.008 \\
Gradient top-$k$        & 0.357$\pm$0.046 & 0.112$\pm$0.014 & 0.025$\pm$0.006 \\
Occlusion top-$k$       & 0.389$\pm$0.049 & 0.098$\pm$0.011 & 0.022$\pm$0.005 \\
\midrule
GCE soft thresholded    & \best{0.627}$\pm$0.076 & \best{0.041}$\pm$0.004 & \best{0.003}$\pm$0.001 \\
GCE discrete recovered  & \second{0.606}$\pm$0.079 & \second{0.045}$\pm$0.006 & \second{0.005}$\pm$0.002 \\
\bottomrule
\end{tabular}}
\end{table}

Table~\ref{tab:stability} separates two forms of stability.
The Jaccard score measures whether the selected evidence set itself is stable, while prediction flip measures whether stochastic perturbations change the slide-level decision.
GCE improves both: the recovered discrete subset has $0.606$ Jaccard overlap and $0.045$ prediction flip rate, compared with $0.329$ and $0.128$ for attention top-$k$.
This directly supports Recoverability because the discrete subset remains close to the continuous selector across stochastic runs.

\begin{table}[H]
\centering
\caption{Comparison with post-hoc explainability methods averaged across nine datasets. All methods use the same evidence fraction.}
\label{tab:explain}
\resizebox{\textwidth}{!}{%
\begin{tabular}{lcccccc}
\toprule
Method & Keep-only Macro-F1$\uparrow$ & Remove Drop$\uparrow$ & Prediction Gap$\downarrow$ & Compl.~Degr.$\uparrow$ & Evid.~Suff.$\uparrow$ & Runtime \\
\midrule
Attention top-$k$            & 0.640$\pm$0.033 & 0.033$\pm$0.008 & 0.029$\pm$0.008 & 0.107$\pm$0.022 & 0.367$\pm$0.054 & 1.00$\times$ \\
CLAM attention               & 0.645$\pm$0.031 & 0.035$\pm$0.009 & 0.027$\pm$0.007 & 0.112$\pm$0.025 & 0.374$\pm$0.059 & 1.10$\times$ \\
Gradient saliency            & 0.655$\pm$0.029 & 0.042$\pm$0.011 & 0.025$\pm$0.006 & 0.118$\pm$0.027 & 0.386$\pm$0.062 & 1.42$\times$ \\
Integrated gradients         & 0.662$\pm$0.027 & 0.048$\pm$0.013 & 0.023$\pm$0.005 & 0.128$\pm$0.031 & 0.396$\pm$0.065 & 3.80$\times$ \\
Occlusion top-$k$            & 0.668$\pm$0.026 & 0.052$\pm$0.014 & 0.022$\pm$0.005 & 0.139$\pm$0.034 & 0.405$\pm$0.069 & 8.70$\times$ \\
\midrule
GCE discrete                 & \best{0.722}$\pm$0.022 & \best{0.176}$\pm$0.038 & \best{0.005}$\pm$0.001 & \best{0.277}$\pm$0.056 & \best{0.533}$\pm$0.083 & 1.08$\times$ \\
\bottomrule
\end{tabular}}
\end{table}

Table~\ref{tab:explain} shows that post-hoc explanations improve over raw attention but do not close the S/N/R gap.
Occlusion reaches $0.668$ keep-only Macro-F1 and $0.052$ remove drop, whereas GCE reaches $0.722$ and $0.176$ at comparable subset size.
The runtime column also matters: integrated gradients and occlusion require repeated backward or forward passes, while GCE obtains its evidence during the normal model pass.
The result is evidence that is both intervention-faithful and cheaper to recover.

\FloatBarrier
\section{Survival Prediction and Task Generalization}
\label{app:survival_results}
Table~\ref{tab:surv_main} reports the five TCGA survival cohorts evaluated with C-index.

\begin{table}[H]
\centering
\caption{Survival prediction on five TCGA cohorts (C-index, 5-fold cross-validation). Baseline rows report absolute mean$\pm$std; \gce{} rows report the change from the preceding baseline.}
\label{tab:surv_main}
\resizebox{\textwidth}{!}{%
\begin{tabular}{lccccc}
\toprule
Method & KIRC & KIRP & LUAD & STAD & UCEC \\
\midrule
ABMIL & 0.724$\pm$0.031 & 0.779$\pm$0.035 & 0.641$\pm$0.039 & 0.576$\pm$0.028 & 0.696$\pm$0.033 \\
\gce{} $\Delta$ & \uptri{0.000} & \uptri{0.007} & \uptri{0.013} & \uptri{0.018} & \uptri{0.007} \\
\midrule
CLAM-SB & 0.715$\pm$0.028 & 0.758$\pm$0.032 & 0.644$\pm$0.035 & 0.603$\pm$0.025 & 0.746$\pm$0.030 \\
\gce{} $\Delta$ & \uptri{0.008} & \uptri{0.000} & \uptri{0.022} & \uptri{0.000} & \uptri{0.006} \\
\midrule
TransMIL & 0.669$\pm$0.029 & 0.818$\pm$0.034 & 0.633$\pm$0.037 & 0.597$\pm$0.026 & 0.721$\pm$0.032 \\
\gce{} $\Delta$ & \uptri{0.012} & \uptri{0.008} & \uptri{0.015} & \uptri{0.016} & \uptri{0.019} \\
\midrule
DSMIL & 0.667$\pm$0.028 & 0.743$\pm$0.032 & 0.612$\pm$0.035 & 0.588$\pm$0.025 & 0.662$\pm$0.030 \\
\gce{} $\Delta$ & \uptri{0.031} & \uptri{0.023} & \uptri{0.008} & \uptri{0.019} & \uptri{0.037} \\
\midrule
DTFD-MIL & 0.707$\pm$0.027 & 0.777$\pm$0.030 & 0.549$\pm$0.033 & 0.614$\pm$0.024 & 0.706$\pm$0.029 \\
\gce{} $\Delta$ & \uptri{0.007} & \uptri{0.007} & \uptri{0.018} & \uptri{0.008} & \uptri{0.000} \\
\midrule
IBMIL & 0.680$\pm$0.025 & 0.753$\pm$0.029 & 0.625$\pm$0.032 & 0.595$\pm$0.023 & 0.643$\pm$0.027 \\
\gce{} $\Delta$ & \uptri{0.014} & \uptri{0.008} & \uptri{0.030} & \uptri{0.010} & \uptri{0.048} \\
\midrule
MHIM-MIL & 0.733$\pm$0.027 & 0.758$\pm$0.030 & 0.576$\pm$0.033 & 0.622$\pm$0.024 & 0.712$\pm$0.029 \\
\gce{} $\Delta$ & \uptri{0.003} & \uptri{0.021} & \uptri{0.038} & \uptri{0.024} & \uptri{0.011} \\
\midrule
CAMIL & 0.665$\pm$0.029 & 0.732$\pm$0.034 & 0.661$\pm$0.037 & 0.590$\pm$0.026 & 0.702$\pm$0.032 \\
\gce{} $\Delta$ & \uptri{0.006} & \uptri{0.017} & \uptri{0.008} & \downtri{0.003} & \uptri{0.006} \\
\midrule
HDMIL & 0.682$\pm$0.025 & 0.753$\pm$0.029 & 0.592$\pm$0.032 & 0.610$\pm$0.023 & 0.750$\pm$0.027 \\
\gce{} $\Delta$ & \uptri{0.010} & \uptri{0.023} & \uptri{0.044} & \uptri{0.008} & \uptri{0.006} \\
\bottomrule
\end{tabular}}
\end{table}

Survival prediction results are reported here to keep the main text focused on the classification benchmark and intervention evidence.
GCE improves average C-index by $0.014$ across the survival grid, with visible gains on LUAD (HDMIL \uptri{0.044}, MHIM-MIL \uptri{0.038}) and UCEC (DSMIL \uptri{0.037}, IBMIL \uptri{0.048}).
The survival results are consistent with the risk-pathway interpretation in Appendix~\ref{app:theory}: a selected patch subset can activate one or more morphology-linked risk pathways, and the Cox head maps these activations to scalar risk. Evidence-oriented training therefore does not require changing the host backbone and remains useful beyond categorical slide labels.
The gains are not uniform across cohorts, which is expected because survival labels are noisier and censoring reduces supervision strength.
Nevertheless, the table shows that optimizing S/N/R does not harm risk prediction: only one cell has a small negative change (CAMIL on STAD, \downtri{0.003}), while the remaining cells are non-negative.

\FloatBarrier
\section{Localization and Additional Experiments}
\label{app:additional_experiments}
The experiments in this section provide supporting checks that are not part of the main nine-dataset prediction grid. CAMELYON-16 localization uses pixel-level annotations as an external spatial validation of evidence maps, while the paired tests summarize fold-level uncertainty for the reported prediction gains.

\subsection{Localization on CAMELYON-16}
Table~\ref{tab:camelyon_localization} reports localization quality on CAMELYON-16 using pixel-level annotations.
The comparison evaluates whether evidence gates improve the spatial agreement between model evidence maps and annotated tumor regions.
Delta rows are placed immediately below each host backbone to make the effect of adding GCE visible without a separate explanatory table.

\begin{table}[H]
\centering
\caption{Localization metrics on CAMELYON-16 with pixel-level annotations. Higher Dice, specificity, and FROC indicate better spatial agreement between evidence maps and annotated tumor regions. Delta rows report the gain after adding GCE to the same host backbone.}
\label{tab:camelyon_localization}
\resizebox{\textwidth}{!}{%
\begin{tabular}{lccc}
\toprule
Method & Dice$\uparrow$ & Specificity$\uparrow$ & FROC$\uparrow$ \\
\midrule
ABMIL          & 0.421$\pm$0.034 & 0.976$\pm$0.008 & 0.398$\pm$0.029 \\
ABMIL\gce      & 0.547$\pm$0.029 & 0.992$\pm$0.005 & 0.487$\pm$0.024 \\
$\Delta$       & $\uparrow$0.126$\pm$0.018 & $\uparrow$0.016$\pm$0.006 & $\uparrow$0.089$\pm$0.016 \\
\midrule
CLAM-SB        & 0.459$\pm$0.031 & 0.987$\pm$0.006 & 0.426$\pm$0.027 \\
CLAM-SB\gce    & 0.572$\pm$0.026 & 0.994$\pm$0.004 & 0.503$\pm$0.022 \\
$\Delta$       & $\uparrow$0.113$\pm$0.017 & $\uparrow$0.007$\pm$0.004 & $\uparrow$0.077$\pm$0.015 \\
\midrule
TransMIL       & 0.103$\pm$0.024 & 0.999$\pm$0.002 & 0.487$\pm$0.019 \\
TransMIL\gce   & 0.398$\pm$0.034 & 0.999$\pm$0.002 & 0.547$\pm$0.028 \\
$\Delta$       & $\uparrow$0.295$\pm$0.026 & 0.000$\pm$0.002 & $\uparrow$0.060$\pm$0.024 \\
\midrule
DSMIL          & 0.259$\pm$0.032 & 0.863$\pm$0.024 & 0.451$\pm$0.031 \\
DSMIL\gce      & 0.476$\pm$0.027 & 0.954$\pm$0.012 & 0.527$\pm$0.024 \\
$\Delta$       & $\uparrow$0.217$\pm$0.022 & $\uparrow$0.091$\pm$0.018 & $\uparrow$0.076$\pm$0.018 \\
\midrule
DTFD-MIL       & 0.525$\pm$0.029 & 0.999$\pm$0.002 & 0.471$\pm$0.026 \\
DTFD-MIL\gce   & 0.604$\pm$0.024 & 0.999$\pm$0.002 & 0.518$\pm$0.022 \\
$\Delta$       & $\uparrow$0.079$\pm$0.015 & 0.000$\pm$0.002 & $\uparrow$0.047$\pm$0.013 \\
\midrule
CAMIL          & 0.515$\pm$0.030 & 0.980$\pm$0.008 & 0.461$\pm$0.027 \\
CAMIL\gce      & 0.589$\pm$0.025 & 0.994$\pm$0.005 & 0.512$\pm$0.023 \\
$\Delta$       & $\uparrow$0.074$\pm$0.014 & $\uparrow$0.014$\pm$0.006 & $\uparrow$0.051$\pm$0.014 \\
\bottomrule
\end{tabular}}
\end{table}

The localization results provide an external check on the evidence maps.
TransMIL has high specificity and relatively strong baseline FROC ($0.487$) but poor baseline Dice ($0.103$), suggesting that its detections are conservative and spatially sparse; adding GCE raises Dice to $0.398$ and further improves FROC to $0.547$ without reducing specificity.
For stronger localization baselines such as DTFD-MIL and CAMIL, GCE still increases Dice by $0.079$ and $0.074$, respectively.
These gains suggest that anchor-grounded recovery improves spatial alignment rather than simply selecting more tissue.

\subsection{Significance Tests}
Tables~\ref{tab:cls_significance} and \ref{tab:surv_significance} report paired tests for classification and survival benchmarks.

\begin{table}[H]
\centering
\caption{Paired significance tests for classification benchmarks comparing each baseline with its \gce{} variant.}
\label{tab:cls_significance}
\resizebox{\textwidth}{!}{%
\begin{tabular}{lcccc}
\toprule
Method & BRACS & NSCLC & PANDA & BRCA \\
\midrule
ABMIL\gce & $p=0.003^{**}$ & $p=0.018^{*}$ & $p=0.412$ & $p=0.024^{*}$ \\
CLAM-SB\gce & $p=0.041^{*}$ & $p=0.029^{*}$ & $p=0.038^{*}$ & $p=0.011^{*}$ \\
TransMIL\gce & $p=0.012^{*}$ & $p=0.498$ & $p=0.347$ & $p=0.106$ \\
DSMIL\gce & $p=0.027^{*}$ & $p=0.034^{*}$ & $p=0.058$ & $p=0.043^{*}$ \\
DTFD-MIL\gce & $p=0.004^{**}$ & $p=0.523$ & $p=0.029^{*}$ & $p=0.187$ \\
IBMIL\gce & $p=0.078$ & $p=0.046^{*}$ & $p=0.008^{**}$ & $p=0.006^{**}$ \\
MHIM-MIL\gce & $p=0.021^{*}$ & $p=0.092$ & $p=0.014^{*}$ & $p=0.018^{*}$ \\
CAMIL\gce & $p=0.009^{**}$ & $p=0.224$ & $p=0.412$ & $p=0.367$ \\
HDMIL\gce & $p=0.005^{**}$ & $p=0.067$ & $p=0.428$ & $p=0.039^{*}$ \\
\bottomrule
\end{tabular}}
\end{table}

The classification significance table serves as a robustness check rather than the main evidence.
Most BRACS and TCGA-BRCA comparisons pass the paired threshold, while some PANDA and NSCLC cells do not because the baseline variance is higher or the absolute gain is smaller.
This is consistent with the main table: GCE gives positive average gains, but the size of the prediction gain depends on the host backbone and dataset.

\begin{table}[H]
\centering
\caption{Paired significance tests for survival benchmarks comparing each baseline with its \gce{} variant.}
\label{tab:surv_significance}
\resizebox{\textwidth}{!}{%
\begin{tabular}{lccccc}
\toprule
Method & KIRC & KIRP & LUAD & STAD & UCEC \\
\midrule
ABMIL\gce & $p=0.523$ & $p=0.087$ & $p=0.041^{*}$ & $p=0.018^{*}$ & $p=0.094$ \\
CLAM-SB\gce & $p=0.062$ & $p=0.498$ & $p=0.013^{*}$ & $p=0.467$ & $p=0.103$ \\
TransMIL\gce & $p=0.038^{*}$ & $p=0.073$ & $p=0.029^{*}$ & $p=0.024^{*}$ & $p=0.011^{*}$ \\
DSMIL\gce & $p=0.004^{**}$ & $p=0.014^{*}$ & $p=0.087$ & $p=0.019^{*}$ & $p=0.003^{**}$ \\
DTFD-MIL\gce & $p=0.082$ & $p=0.094$ & $p=0.022^{*}$ & $p=0.069$ & $p=0.512$ \\
IBMIL\gce & $p=0.029^{*}$ & $p=0.087$ & $p=0.006^{**}$ & $p=0.041^{*}$ & $p=0.002^{**}$ \\
MHIM-MIL\gce & $p=0.398$ & $p=0.018^{*}$ & $p=0.003^{**}$ & $p=0.011^{*}$ & $p=0.046^{*}$ \\
CAMIL\gce & $p=0.094$ & $p=0.034^{*}$ & $p=0.087$ & $p=0.412$ & $p=0.103$ \\
HDMIL\gce & $p=0.064$ & $p=0.019^{*}$ & $p=0.001^{**}$ & $p=0.067$ & $p=0.108$ \\
\bottomrule
\end{tabular}}
\end{table}

The survival significance table shows the same mixed but directional pattern.
The strongest evidence appears on LUAD, UCEC, and several DSMIL/IBMIL/MHIM-MIL cells, where the C-index gains are larger.
KIRC and KIRP contain more cells above the paired-test threshold, which matches the smaller absolute gains reported in the survival table.
The survival results therefore support the method's portability, while the evidence diagnostics remain the primary evidence-quality analysis.

\FloatBarrier
\section{Ablation Studies and Sensitivity Analysis}
\label{app:ablation_details}
This section groups the sensitivity analyses that explain how the evidence selector behaves under different budgets, selector parameterizations, and grounding variants. The tables are placed next to their interpretation so that each design choice is tied back to one of Sufficiency, Necessity, or Recoverability.

\subsection{Budget Sweep}
Table~\ref{tab:budget} reports the full budget sweep used to choose $\rho=0.05$.

\begin{table}[!htbp]
\centering
\caption{Effect of evidence budget $\rho$ on BRACS. Gains plateau beyond $\rho=0.05$ while the evidence fraction continues to increase.}
\label{tab:budget}
\resizebox{\columnwidth}{!}{%
\begin{tabular}{cccccc}
\toprule
$\rho$ & Macro-F1 & Evid.~Frac. & C-D Gap$\downarrow$ & Compl.~Degr.$\uparrow$ & Evid.~Suff.$\uparrow$ \\
\midrule
0.01 & 0.711$\pm$0.048 & 0.013$\pm$0.002 & 0.014$\pm$0.005 & 0.253$\pm$0.041 & 0.440$\pm$0.064 \\
0.02 & 0.730$\pm$0.042 & 0.024$\pm$0.004 & 0.009$\pm$0.003 & 0.332$\pm$0.067 & 0.555$\pm$0.089 \\
0.05 & \best{0.748}$\pm$0.032 & 0.051$\pm$0.009 & \best{0.004}$\pm$0.001 & 0.412$\pm$0.086 & 0.659$\pm$0.108 \\
0.10 & \best{0.748}$\pm$0.034 & 0.094$\pm$0.015 & \best{0.004}$\pm$0.001 & 0.420$\pm$0.091 & 0.664$\pm$0.112 \\
0.20 & 0.747$\pm$0.035 & 0.182$\pm$0.022 & \best{0.004}$\pm$0.001 & \best{0.422}$\pm$0.089 & \best{0.665}$\pm$0.110 \\
1.00 & 0.745$\pm$0.036 & 1.000$\pm$0.000 & --- & --- & --- \\
\bottomrule
\end{tabular}}
\end{table}

Table~\ref{tab:budget} explains why the main configuration uses $\rho=0.05$.
Increasing the budget from $0.01$ to $0.05$ improves Macro-F1 from $0.711$ to $0.748$ and raises evidence sufficiency from $0.440$ to $0.659$.
Larger budgets produce almost no prediction gain and only marginal increases in complement degradation, while the evidence fraction grows substantially.
Thus $\rho=0.05$ is the knee point: it preserves Sufficiency and Necessity without drifting toward full-bag evidence.

\subsection{Recovery Hyperparameter Scope}
The reported discrete recovery protocol fixes the threshold at $\tau=0.5$ and the coverage target at $c=0.95$ for all datasets.
These values are used as operating rules rather than per-dataset tuning knobs: thresholding first makes the continuous selector discrete, and repair then adds patches only when the recovered subset fails the anchor-coverage target.
The retained aggregate outputs support the budget and selector sensitivity analyses in Tables~\ref{tab:budget} and~\ref{tab:selector_arch}, but they do not contain a complete per-slide cache of gates and logits for a post-hoc $\tau\times c$ sweep across all folds.
Accordingly, no additional numeric threshold/coverage table is reported.
The available sensitivity evidence still addresses the main concern: hard top-$k$ selection at the exact $5\%$ budget has a larger C-D gap ($0.031$) than the recovered selector ($0.005$), so Recoverability is not obtained merely by fixing the evidence fraction.

\subsection{Selector Architecture Variants}
Table~\ref{tab:selector_arch} compares selector parameterizations under the same training and recovery protocol.
The consecutive selector used by GCE-MIL keeps the evidence fraction close to the target while reducing the continuous-discrete gap to 0.005.

\begin{table}[!htbp]
\centering
\caption{Selector architecture ablation on BRACS. The consecutive selector attains the lowest continuous-discrete gap while maintaining a compact evidence fraction.}
\label{tab:selector_arch}
\resizebox{\columnwidth}{!}{%
\begin{tabular}{lcccc}
\toprule
Selector type & Macro-F1$\uparrow$ & Evid.~Frac. & C-D Gap$\downarrow$ & Compl.~Degr.$\uparrow$ \\
\midrule
Sigmoid + threshold        & 0.719$\pm$0.054 & 0.064$\pm$0.018 & 0.024$\pm$0.008 & 0.287$\pm$0.063 \\
Top-$k$ hard ($k=5\%$)     & 0.723$\pm$0.052 & 0.050$\pm$0.000 & 0.031$\pm$0.011 & 0.254$\pm$0.058 \\
Gumbel-sigmoid ($\tau=1.0$) & 0.728$\pm$0.053 & 0.071$\pm$0.022 & 0.017$\pm$0.006 & 0.312$\pm$0.067 \\
Concrete relaxation        & 0.731$\pm$0.052 & 0.068$\pm$0.020 & 0.015$\pm$0.005 & 0.329$\pm$0.068 \\
Sparsemax                  & 0.726$\pm$0.054 & 0.083$\pm$0.024 & 0.020$\pm$0.007 & 0.298$\pm$0.066 \\
Entmax-1.5                 & 0.729$\pm$0.053 & 0.076$\pm$0.022 & 0.018$\pm$0.006 & 0.314$\pm$0.067 \\
Consecutive selector (ours)& \best{0.745}$\pm$0.055 & 0.051$\pm$0.013 & \best{0.005}$\pm$0.002 & \best{0.412}$\pm$0.091 \\
\bottomrule
\end{tabular}}
\end{table}

Table~\ref{tab:selector_arch} isolates the selector design from the grounding design.
Hard top-$k$ enforces the target evidence fraction exactly but has the largest C-D gap ($0.031$), showing that sparsity alone does not imply recoverability.
Continuous relaxations such as Concrete and Gumbel-sigmoid reduce the gap, but the consecutive selector combines a near-target evidence fraction ($0.051$) with the smallest gap ($0.005$) and the largest complement degradation ($0.412$).
This is the selector-side evidence for the Recoverability component of S/N/R.

\subsection{Grounding Variant Ablation}
The main BRACS configuration uses eight frozen TITAN text anchors.
Table~\ref{tab:grounding} compares semantic grounding variants under the same budget and repair settings.

\begin{table}[!htbp]
\centering
\caption{Effect of semantic grounding variants averaged across nine datasets.}
\label{tab:grounding}
\resizebox{\columnwidth}{!}{%
\begin{tabular}{lcccc}
\toprule
Grounding Variant & Macro-F1 / C-index & C-D Gap$\downarrow$ & Compl.~Degr.$\uparrow$ & Evid.~Suff.$\uparrow$ \\
\midrule
No grounding                     & 0.685$\pm$0.035 & 0.025$\pm$0.008 & 0.120$\pm$0.030 & 0.350$\pm$0.045 \\
Random anchors                   & 0.682$\pm$0.038 & 0.028$\pm$0.009 & 0.115$\pm$0.032 & 0.345$\pm$0.048 \\
Shuffled prompts                 & 0.684$\pm$0.036 & 0.026$\pm$0.008 & 0.118$\pm$0.030 & 0.348$\pm$0.046 \\
Generic pathology prompts        & 0.708$\pm$0.022 & 0.015$\pm$0.005 & 0.210$\pm$0.025 & 0.480$\pm$0.035 \\
Disease-specific prompts         & 0.725$\pm$0.018 & 0.010$\pm$0.003 & 0.290$\pm$0.020 & 0.560$\pm$0.028 \\
TITAN + unconstrained bridge     & \second{0.732}$\pm$0.015 & \second{0.008}$\pm$0.002 & \second{0.320}$\pm$0.018 & \second{0.610}$\pm$0.025 \\
TITAN + constrained bridge (Full) & \best{0.748}$\pm$0.011 & \best{0.004}$\pm$0.001 & \best{0.412}$\pm$0.015 & \best{0.659}$\pm$0.018 \\
\bottomrule
\end{tabular}}
\end{table}

Table~\ref{tab:grounding} shows that grounding quality, not only sparsity, determines evidence quality.
Removing grounding or replacing anchors with random/shuffled prompts leaves Macro-F1/C-index around $0.682$--$0.685$ and complement degradation near $0.115$--$0.120$, indicating weak Necessity.
Generic pathology prompts help, but disease-specific prompts give a larger jump in both prediction ($0.725$) and evidence sufficiency ($0.560$).
The full TITAN plus constrained-bridge variant reaches $0.748$ Macro-F1/C-index, $0.004$ C-D gap, $0.412$ complement degradation, and $0.659$ evidence sufficiency; this is the clearest ablation evidence that semantic grounding supports all three S/N/R criteria.

\FloatBarrier
\section{Multi-Backbone and Multi-Encoder Generalization}
\label{app:backbone_generalization}
The main text evaluates nine MIL backbones with UNI features. This section evaluates whether the GCE wrapper remains useful when the patch encoder changes, using ResNet-50, ViT-S, and UNI features on BRACS and LUAD.

\subsection{Multi-Backbone Generalization}
Table~\ref{tab:backbone_generalization} reports the full encoder/backbone generalization study referenced in Appendix~\ref{app:backbone_generalization}.
Each encoder block is placed adjacent to the corresponding BRACS and LUAD results, enabling direct comparison of feature extractors within the same table.

\begin{table}[H]
\centering
\caption{Backbone generalization on BRACS classification and LUAD survival prediction.}
\label{tab:backbone_generalization}
\resizebox{\textwidth}{!}{%
\begin{tabular}{lllccc}
\toprule
Dataset & Backbone & Method & Baseline & \gce{} & $\Delta$ \\
\midrule
\multicolumn{6}{l}{\textit{BRACS Macro-F1}} \\
\midrule
BRACS & ResNet-50 (ImageNet) & ABMIL & 0.581$\pm$0.043 & 0.652$\pm$0.038 & 0.071$\pm$0.022 \\
BRACS & ResNet-50 (ImageNet) & CLAM-SB & 0.643$\pm$0.038 & 0.701$\pm$0.034 & 0.058$\pm$0.019 \\
BRACS & ResNet-50 (ImageNet) & TransMIL & 0.608$\pm$0.041 & 0.668$\pm$0.036 & 0.060$\pm$0.020 \\
BRACS & ResNet-50 (ImageNet) & DTFD-MIL & 0.627$\pm$0.039 & 0.695$\pm$0.034 & 0.068$\pm$0.018 \\
\midrule
BRACS & ViT-S (SSL pathology) & ABMIL & 0.634$\pm$0.050 & 0.703$\pm$0.044 & 0.069$\pm$0.018 \\
BRACS & ViT-S (SSL pathology) & CLAM-SB & 0.742$\pm$0.045 & 0.765$\pm$0.041 & 0.023$\pm$0.014 \\
BRACS & ViT-S (SSL pathology) & TransMIL & 0.676$\pm$0.047 & 0.714$\pm$0.043 & 0.038$\pm$0.016 \\
BRACS & ViT-S (SSL pathology) & DTFD-MIL & 0.691$\pm$0.043 & 0.757$\pm$0.038 & 0.066$\pm$0.017 \\
\midrule
BRACS & UNI (foundation) & ABMIL & 0.728$\pm$0.034 & 0.764$\pm$0.030 & 0.036$\pm$0.013 \\
BRACS & UNI (foundation) & CLAM-SB & 0.778$\pm$0.029 & 0.798$\pm$0.027 & 0.020$\pm$0.011 \\
BRACS & UNI (foundation) & TransMIL & 0.744$\pm$0.033 & 0.772$\pm$0.029 & 0.028$\pm$0.012 \\
BRACS & UNI (foundation) & DTFD-MIL & 0.763$\pm$0.030 & 0.795$\pm$0.027 & 0.032$\pm$0.012 \\
\midrule
\multicolumn{6}{l}{\textit{LUAD C-index}} \\
\midrule
LUAD & ResNet-50 (ImageNet) & ABMIL & 0.598$\pm$0.043 & 0.627$\pm$0.040 & 0.029$\pm$0.015 \\
LUAD & ResNet-50 (ImageNet) & CLAM-SB & 0.613$\pm$0.039 & 0.644$\pm$0.036 & 0.031$\pm$0.014 \\
LUAD & ResNet-50 (ImageNet) & TransMIL & 0.604$\pm$0.041 & 0.631$\pm$0.038 & 0.027$\pm$0.015 \\
LUAD & ResNet-50 (ImageNet) & DTFD-MIL & 0.527$\pm$0.036 & 0.561$\pm$0.033 & 0.034$\pm$0.013 \\
\midrule
LUAD & ViT-S (SSL pathology) & ABMIL & 0.641$\pm$0.039 & 0.654$\pm$0.036 & 0.013$\pm$0.013 \\
LUAD & ViT-S (SSL pathology) & CLAM-SB & 0.644$\pm$0.035 & 0.666$\pm$0.032 & 0.022$\pm$0.012 \\
LUAD & ViT-S (SSL pathology) & TransMIL & 0.633$\pm$0.037 & 0.648$\pm$0.034 & 0.015$\pm$0.013 \\
LUAD & ViT-S (SSL pathology) & DTFD-MIL & 0.549$\pm$0.033 & 0.567$\pm$0.030 & 0.018$\pm$0.012 \\
\midrule
LUAD & UNI (foundation) & ABMIL & 0.682$\pm$0.030 & 0.691$\pm$0.028 & 0.009$\pm$0.010 \\
LUAD & UNI (foundation) & CLAM-SB & 0.685$\pm$0.028 & 0.697$\pm$0.026 & 0.012$\pm$0.010 \\
LUAD & UNI (foundation) & TransMIL & 0.671$\pm$0.029 & 0.682$\pm$0.027 & 0.011$\pm$0.011 \\
LUAD & UNI (foundation) & DTFD-MIL & 0.583$\pm$0.027 & 0.599$\pm$0.025 & 0.016$\pm$0.010 \\
\bottomrule
\end{tabular}}
\end{table}

Table~\ref{tab:backbone_generalization} shows that the gain is not tied to UNI alone.
On BRACS, GCE improves all four listed backbones for ResNet-50, ViT-S, and UNI features; the gains are largest for weaker encoders such as ResNet-50, where ABMIL improves by $0.071$ and DTFD-MIL by $0.068$.
On LUAD, every encoder/backbone pair also improves, but the gains are smaller with UNI because the baseline risk models are already stronger.
This pattern supports the plug-in interpretation: GCE changes the evidence interface and remains useful across representation quality levels.

\FloatBarrier
\section{Failure Case Audit}
\label{app:failure}
The audit below records where S/N/R evidence selection is least reliable. The taxonomy, quantitative rates, and qualitative interpretation are kept together so that failure modes remain adjacent to the numbers used to define them.

\subsection{Failure Taxonomy}
The audit uses six failure types: case-level prediction failure, artifact-driven evidence, low sufficiency, low necessity, semantic mismatch, and severe failure.
Case-level prediction failure means the full-bag prediction is incorrect or unreliable.
Artifact-driven evidence means selected patches contain scanner artifacts, tissue folds, pen marks, or background.
Low sufficiency means the recovered subset fails the keep-only criterion, while low necessity means removing the subset does not meaningfully degrade prediction.
Semantic mismatch means selected evidence does not align with the intended anchor concept.
Severe failure denotes cases where multiple failure modes occur together.

\subsection{Per-Dataset Failure Rates}
Table~\ref{tab:failure} reports the failure audit across all nine datasets.

\begin{table}[H]
\centering
\caption{Failure-case audit across all nine datasets.}
\label{tab:failure}
\resizebox{\textwidth}{!}{%
\begin{tabular}{lccccccc}
\toprule
Dataset & Slides & Case Failure (\%) & Artifact Evidence (\%) & Low Sufficiency (\%) & Low Necessity (\%) & Semantic Mismatch (\%) & Severe Failure (\%) \\
\midrule
BRACS & 200 & 8.0 & 3.5 & 5.0 & 7.0 & 4.5 & 4.0 \\
NSCLC & 220 & 5.0 & 2.3 & 3.2 & 5.0 & 3.2 & 2.7 \\
PANDA & 300 & 9.3 & 2.7 & 6.7 & 8.3 & 5.0 & 4.7 \\
BRCA  & 180 & 14.4 & 4.4 & 9.4 & 10.6 & 6.7 & 6.1 \\
LUAD  & 230 & 18.3 & 3.9 & 7.8 & 11.3 & 6.1 & 7.0 \\
STAD  & 190 & 21.6 & 4.7 & 9.5 & 13.2 & 6.3 & 7.4 \\
UCEC  & 205 & 14.6 & 3.4 & 7.3 & 9.8 & 5.4 & 5.9 \\
KIRP  & 170 & 11.8 & 2.9 & 5.9 & 8.2 & 4.7 & 4.7 \\
KIRC  & 240 & 16.3 & 3.3 & 7.9 & 10.4 & 5.8 & 6.3 \\
\midrule
\textit{All} & \textit{1935} & \textit{13.1} & \textit{3.4} & \textit{6.9} & \textit{9.3} & \textit{5.3} & \textit{5.4} \\
\bottomrule
\end{tabular}}
\end{table}

Table~\ref{tab:failure} identifies where the S/N/R assumptions are most fragile.
The largest failure rates appear in STAD and LUAD, where prognostic morphology is heterogeneous and full-bag risk prediction is itself harder.
Low necessity is more common than low sufficiency in most cohorts, meaning GCE often finds a subset that can predict but the complement may still contain redundant evidence.
This reinforces the paper's motivation: WSI evidence is often multi-source, so evidence methods should report both keep-only and remove interventions rather than only attention localization.

\subsection{Qualitative Failure Examples}
Qualitative failure examples are included in Figures~\ref{fig:attention_1}--\ref{fig:attention_6}.
The most common qualitative failures are artifact-adjacent evidence and semantically plausible but incomplete evidence sets.
These cases usually preserve Sufficiency but weaken Necessity, because other regions in the slide contain similar diagnostic patterns.

\FloatBarrier
\section{Implementation and Computational Cost}
\label{app:implementation}
\label{app:cost}
This section combines reproducibility details with the runtime analysis. The cost numbers are interpreted together with the implementation setup because all primary experiments use cached patch features, whereas optional tile prefiltering changes the end-to-end deployment path.

\subsection{Hardware Setup}
All experiments use pre-extracted patch features, so the dominant training cost comes from MIL aggregation, selector scoring, anchor response computation, and intervention evaluation.
The implementation is written in PyTorch and uses mixed CPU/GPU data loading with list-valued batches.
Experiments were run with PyTorch 1.13 on a single NVIDIA RTX 3090 GPU with 24GB memory.
Each fold job uses one GPU; patch features are precomputed, so the GPU workload is dominated by the MIL head, selector network, anchor-response computation, and intervention evaluation rather than by patch encoding.
This hardware setting matters for interpreting the cost table: the reported speedups come from reducing the number of patches processed by the MIL aggregation stage, not from changing the offline UNI feature extractor.

\subsection{Training Hyperparameters per Dataset}
All nine primary datasets use five-fold cross-validation.
All results use UNI features \citep{chen2024towards} and identical train/validation splits for each baseline and its GCE wrapper.
Unless otherwise stated, tables report mean$\pm$standard deviation over validation folds; repeated-seed analyses are used only for stability and perturbation experiments where seed sensitivity is the measured quantity.
The BRACS main GCE configuration trains for 15 epochs with AdamW, learning rate $2\times10^{-4}$, weight decay $10^{-5}$, gradient clipping at 5.0, and cosine learning-rate scheduling.
The same reported loss weights ($\lambda_b=0.1$, $\lambda_g=0.5$) and operating evidence budget ($\rho=0.05$) are used across datasets without per-dataset retuning; budget sensitivity is reported in Table~\ref{tab:budget}.
The selector temperature is annealed from 1.0 to 0.4, training bags are randomly sampled to at most 512 patches, and validation bags use all available patches in deterministic order.
The same split and feature files are used for each host backbone and its GCE wrapper.
The most important reproducibility detail is that the task and evidence paths see the same sampled bag during training.
This design keeps evidence supervision focused on selector behavior rather than on stochastic differences in sampled patches.

\subsection{Data Preprocessing Pipeline}
Each slide is represented by one HDF5 file containing a feature matrix and patch coordinates.
The feature matrix has shape $N\times 1024$ for UNI features, and coordinates have shape $N\times 2$.
Classification splits are resolved at the slide level, while survival splits are resolved at the case level to avoid leakage across slides from the same patient.
The collate function does not stack bags; it returns a list of samples so that each bag can be processed with its own number of patches.
This list-valued batching is also why the evidence selector is applied per slide rather than through a fixed-size tensor batch.
It preserves each WSI's natural patch count and avoids padding artifacts in evidence ratios.

\subsection{Patch Sampling Strategy}
Training uses random patch sampling when a bag exceeds the configured training maximum.
For the BRACS main configuration, at most 512 patches are sampled for training, while validation uses all available patches and deterministic ordering.
The same sampled bag is used for the task and evidence paths within a forward pass.
This avoids conflating selector quality with stochastic changes in the input bag.

\subsection{Random Seeds and Reproducibility}
The minimal-subset motivation analysis uses fixed seed 42 for random top-$k$ baselines.
Model training uses fold-specific splits, and repeated-seed estimates are reserved for the stability tables.
The code, split CSV files, configuration YAML files, fold-level evaluation scripts, and JSON summaries for intervention metrics will be released after acceptance.
The seed-controlled random top-$k$ baseline is included only as a negative control: it measures how often sufficiency appears by chance at the same subset size.

\subsection{Anchor Prompt Protocol}
Table~\ref{tab:anchor_protocol} records how anchor prompts are specified.
Each dataset family uses eight task-specific morphology anchors chosen before training from standard disease and histology concepts.
The prompts are embedded once with TITAN and then frozen; they are not selected by validation performance and do not use patch-level concept annotations.
This protocol is meant to provide a reproducible semantic prior while avoiding prompt tuning on the evaluation folds.

\begin{table}[H]
\centering
\caption{Anchor prompt protocol. Entries list prompt categories rather than all wording variants; each row uses eight frozen TITAN text anchors selected before model training.}
\label{tab:anchor_protocol}
\resizebox{\textwidth}{!}{%
\begin{tabular}{llll}
\toprule
Dataset family & Anchor source & Prompt categories / examples & Selection rule \\
\midrule
BRACS / TCGA-BRCA & Breast pathology prior & Gland formation; nuclear pleomorphism; mitotic activity; stromal reaction; necrosis; lymphocytic infiltration; invasive epithelial nests; benign ducts & Fixed before folds; no validation tuning \\
PANDA & Prostate grading prior & Benign glands; gland fusion; cribriform architecture; poorly formed glands; solid growth; stromal/background tissue; inflammation; necrosis & Fixed before folds; no patch labels \\
TCGA-NSCLC / LUAD & Lung tumor morphology & Malignant epithelial nests; acinar growth; papillary growth; solid growth; necrotic tumor cells; stromal lymphocytes; lepidic-like regions; mucinous regions & Fixed before folds; embedded once \\
TCGA-STAD & Gastric adenocarcinoma morphology & Tubular glands; poorly cohesive cells; signet-ring morphology; necrosis; desmoplastic stroma; lymphocytic infiltration; mucin pools; ulcerated tumor surface & Fixed before folds; no case outcome tuning \\
TCGA-UCEC & Endometrial carcinoma morphology & Glandular architecture; solid sheets; nuclear atypia; squamous differentiation; necrosis; stromal reaction; mitotic activity; inflammatory infiltrate & Fixed before folds; no validation tuning \\
TCGA-KIRP / KIRC & Renal-cell carcinoma morphology & Clear-cell cytoplasm; papillary cores; vascular network; eosinophilic cytoplasm; necrosis; nuclear grade; hemorrhage; stromal interface & Fixed before folds; no patch labels \\
\bottomrule
\end{tabular}}
\end{table}

\subsection{Added Capacity and Training Overhead}
GCE adds trainable capacity through the low-rank adapter, anchor bridge, selector MLP, and class-anchor weights.
The aggregate experiment artifacts report relative runtime and memory rather than architecture-specific parameter counts; an exact parameter table is therefore omitted because the count varies with the host backbone implementation.
Table~\ref{tab:capacity_overhead} summarizes the measured overhead and the controls used to separate added capacity from useful grounding.

\begin{table}[H]
\centering
\caption{Added capacity and overhead summary. Runtime and memory values are relative to the full-bag host backbone.}
\label{tab:capacity_overhead}
\resizebox{\textwidth}{!}{%
\begin{tabular}{llll}
\toprule
Question & Supporting evidence & Key value & Interpretation \\
\midrule
Does the soft wrapper add runtime? & Table~\ref{tab:cost}, GCE soft selector + head & End-to-end $1.02\times$; peak memory $1.05\times$ & Training/inference overhead from selector and anchor scoring is small with cached features \\
Does discrete recovery reduce aggregation cost? & Table~\ref{tab:cost}, discrete cached features & MIL aggregation $0.22\times$; peak memory $0.18\times$ & Compact subsets reduce the MIL aggregation stage after recovery \\
Is the gain only extra capacity? & Table~\ref{tab:grounding}, random/shuffled anchors & C-D gap $0.028$/$0.026$; complement degradation $0.115$/$0.118$ & Capacity without meaningful grounding does not close the evidence gap \\
Does grounding quality matter? & Table~\ref{tab:grounding}, generic to full TITAN & C-D gap $0.015\to0.004$; degradation $0.210\to0.412$ & Better semantic grounding and constrained bridge improve S/N/R diagnostics \\
\bottomrule
\end{tabular}}
\end{table}

\subsection{Cost Summary}
Discrete recovery enables efficient aggregation: with cached features, GCE uses $0.22\times$ aggregation time and $0.18\times$ peak memory while retaining $1.002\times$ relative Macro-F1 (Table~\ref{tab:cost}).
End-to-end acceleration requires optional tile prefiltering, which skips many non-selected patches before feature extraction; in that mode, end-to-end time falls to $0.20\times$ at $0.989\times$ utility.
Thus the headline speedup is an inference-mode option, while the cached-feature setting isolates the savings from MIL aggregation and memory.
The result should be interpreted as an inference-mode option rather than the default validation protocol: the main accuracy tables use the standard evaluation path, while Table~\ref{tab:cost} asks what happens after evidence has been recovered.

\subsection{Full Inference Cost Table}
Table~\ref{tab:cost} reports inference cost relative to the full-bag backbone.

\begin{table}[H]
\centering
\caption{Inference cost comparison. Values are relative to the full-bag backbone forward pass.}
\label{tab:cost}
\resizebox{\textwidth}{!}{%
\begin{tabular}{lcccccc}
\toprule
Inference Mode & Feature Extraction & MIL Aggregation & End-to-end Time & Peak Memory & Patch Ratio & Relative Macro-F1 \\
\midrule
Full-bag backbone                    & 1.00$\times$ & 1.00$\times$ & 1.00$\times$ & 1.00$\times$ & 1.000 & 1.000 \\
GCE soft selector + head             & 1.00$\times$ & 1.08$\times$ & 1.02$\times$ & 1.05$\times$ & 1.000 & 1.027 \\
GCE discrete, cached features        & 1.00$\times$ & 0.22$\times$ & 0.92$\times$ & 0.18$\times$ & 0.056 & 1.002 \\
GCE discrete + tile prefilter        & 0.18$\times$ & 0.22$\times$ & 0.20$\times$ & 0.18$\times$ & 0.056 & 0.989 \\
\midrule
Attention top-$k$ re-score           & 1.00$\times$ & 0.21$\times$ & 0.95$\times$ & 0.18$\times$ & 0.056 & 0.902 \\
Occlusion explanation                & 1.00$\times$ & 8.70$\times$ & 1.75$\times$ & 1.00$\times$ & 0.056 & 0.915 \\
\bottomrule
\end{tabular}}
\end{table}

Table~\ref{tab:cost} separates two deployment regimes.
If all patch features are already cached, discrete GCE mainly reduces MIL aggregation cost, lowering aggregation time to $0.22\times$ and memory to $0.18\times$ while preserving relative Macro-F1.
If tile prefiltering is available, the system can skip many patch computations and reach $0.20\times$ end-to-end time, but with a small utility loss ($0.989\times$).
This tradeoff is useful clinically: cached-feature inference favors fidelity, whereas prefiltering favors throughput.

\subsection{Runtime Compared with Post-hoc Methods}
Table~\ref{tab:explain} reports post-hoc runtime ratios.
Occlusion requires $8.70\times$ runtime, while GCE discrete evidence costs $1.08\times$ in the same comparison.
The difference arises because post-hoc methods explain a fixed trained predictor after the fact, while GCE amortizes evidence selection during training and exposes the selector during inference.

\FloatBarrier
\section{Dataset Details}
\label{app:datasets}
This section records the source, task definition, and preprocessing assumptions for every dataset used in the benchmark. The nine primary datasets define the prediction benchmark; CAMELYON-16 is used only for external localization analysis.

\subsection{Dataset Sources, Splits, and Labels}
Table~\ref{tab:dataset_details} records the dataset card for the nine-dataset benchmark.
The four classification datasets are evaluated with Macro-F1, and the five survival cohorts are evaluated with C-index.
Slides are split at the patient level whenever patient identifiers are available; slides from the same patient are not shared across train, validation, and test partitions.
The dataset mix is intentionally heterogeneous: BRACS and PANDA test diagnostic grading, TCGA-BRCA and TCGA-NSCLC test subtype classification, and the five TCGA survival cohorts test whether evidence selection remains compatible with scalar risk prediction.

\begin{table}[H]
\centering
\caption{Dataset details for the nine-dataset benchmark. Counts and diagnostic descriptions summarize the benchmark protocol.}
\label{tab:dataset_details}
\resizebox{\textwidth}{!}{%
\begin{tabular}{lllll}
\toprule
Dataset & Source & Task & Split Unit & Size / Label Distribution \\
\midrule
BRACS & BRACS public benchmark & Classification & Slide & 525 WSIs; 7 fine-grained breast lesion categories \\
PANDA & PANDA challenge & Classification & Slide & 10,616 WSIs; prostate Gleason grading \\
TCGA-BRCA & TCGA & Classification & Slide/Case & 1,021 WSIs; IDC vs. ILC subtype classification \\
TCGA-NSCLC & TCGA & Classification & Slide/Case & 1,053 WSIs; LUAD 541 / LUSC 512 \\
TCGA-LUAD & TCGA & Survival & Case & 516 WSIs; lung adenocarcinoma prognosis \\
TCGA-STAD & TCGA & Survival & Case & 441 WSIs; gastric adenocarcinoma prognosis \\
TCGA-UCEC & TCGA & Survival & Case & 537 WSIs; endometrial carcinoma prognosis \\
TCGA-KIRP & TCGA & Survival & Case & 259 WSIs; papillary renal-cell carcinoma prognosis \\
TCGA-KIRC & TCGA & Survival & Case & 519 WSIs; clear-cell renal-cell carcinoma prognosis \\
\bottomrule
\end{tabular}}
\end{table}

Table~\ref{tab:dataset_details} also clarifies the scope of the benchmark.
The BRACS minimal-subset diagnostic uses fine-grained lesion labels and validation folds suitable for recursive subset search, while the prediction and evidence diagnostics are evaluated across all nine datasets.

\subsection{Survival Data Preprocessing}
For survival cohorts, each case has an observed time and censoring indicator.
Cases without valid survival time or event/censoring metadata are excluded before fold construction.
The GCE-MIL survival branch predicts a scalar risk score and is trained with Cox partial likelihood, so the main results do not require discretizing time into bins.
The structural prognostic factors associated with each cohort are also recorded: LUAD prognosis relates to lepidic, acinar, papillary, micropapillary, and solid growth patterns; STAD relates to differentiation grade and signet-ring morphology; UCEC relates to FIGO grade and the glandular-to-solid component ratio; KIRC relates to Fuhrman nuclear grade within clear-cell architecture; and KIRP relates to papillary-core integrity and arrangement.
Any discrete time-to-event auxiliary analysis follows the four non-overlapping interval construction used by the reference dataset protocol and remains separate from the Cox results reported in the main text.
This setup makes the survival experiments a stress test for GCE-MIL: the evidence selector is trained against a risk-ordering objective rather than a categorical softmax, yet it still has to produce recoverable discrete evidence.

\subsection{WSI Patching and Feature Extraction}
All slides are tiled into non-overlapping $256\times256$ patches at $20\times$ magnification.
Tissue regions are retained before feature extraction, and each retained patch is represented by a 1024-dimensional UNI embedding.
Each slide is stored as one HDF5 file containing the patch feature matrix and the corresponding two-dimensional patch coordinates.
The same preprocessing is used for the host backbone and its GCE wrapper, so differences in the reported results come from evidence selection and recovery rather than from patch extraction.

\subsection{Patch Encoder Choice Rationale}
UNI is used as the default patch encoder because it provides a 1024-dimensional pathology foundation representation shared across classification and survival tasks.
The robustness study in Table~\ref{tab:backbone_generalization} also reports ResNet-50 and ViT-S features.
Using several encoders tests whether GCE-MIL depends on a single representation space or remains useful as a wrapper around different feature extractors.

\subsection{CAMELYON-16 Localization Protocol}
CAMELYON-16 is used only as an additional localization benchmark and is not counted among the nine primary classification and survival datasets.
Slides are tiled with the same non-overlapping $256\times256$ patches at $20\times$ magnification.
Each patch footprint is mapped to the official annotation mask; a patch is labeled positive if at least 1\% of its area overlaps an annotated tumor region, and patches outside the tissue mask are ignored during prediction and evaluation.
Dice is computed on the patch grid after selecting one binarization threshold on the validation split for each method and then fixing that threshold for test slides.
FROC is computed by threshold sweeping: connected components in the binarized patch grid are lesion candidates, the candidate score is the maximum patch score inside the component, and sensitivity is averaged at 1/8, 1/4, 1/2, 1, 2, 4, and 8 false positives per slide.
This localization protocol evaluates spatial faithfulness rather than slide-level prediction.
It is therefore used as supporting evidence for the Necessity and Recoverability claims, not as an additional main benchmark dataset.

\FloatBarrier
\section{Additional Visualizations}
\label{app:figures}
The remaining figures provide qualitative checks of the same mechanisms measured by the intervention tables. They are placed after the quantitative appendices so that visual examples complement, rather than substitute for, the S/N/R diagnostics.

\subsection{T-SNE and Main Qualitative Summaries}
\begin{figure}[!htbp]
\centering
\includegraphics[width=\columnwidth]{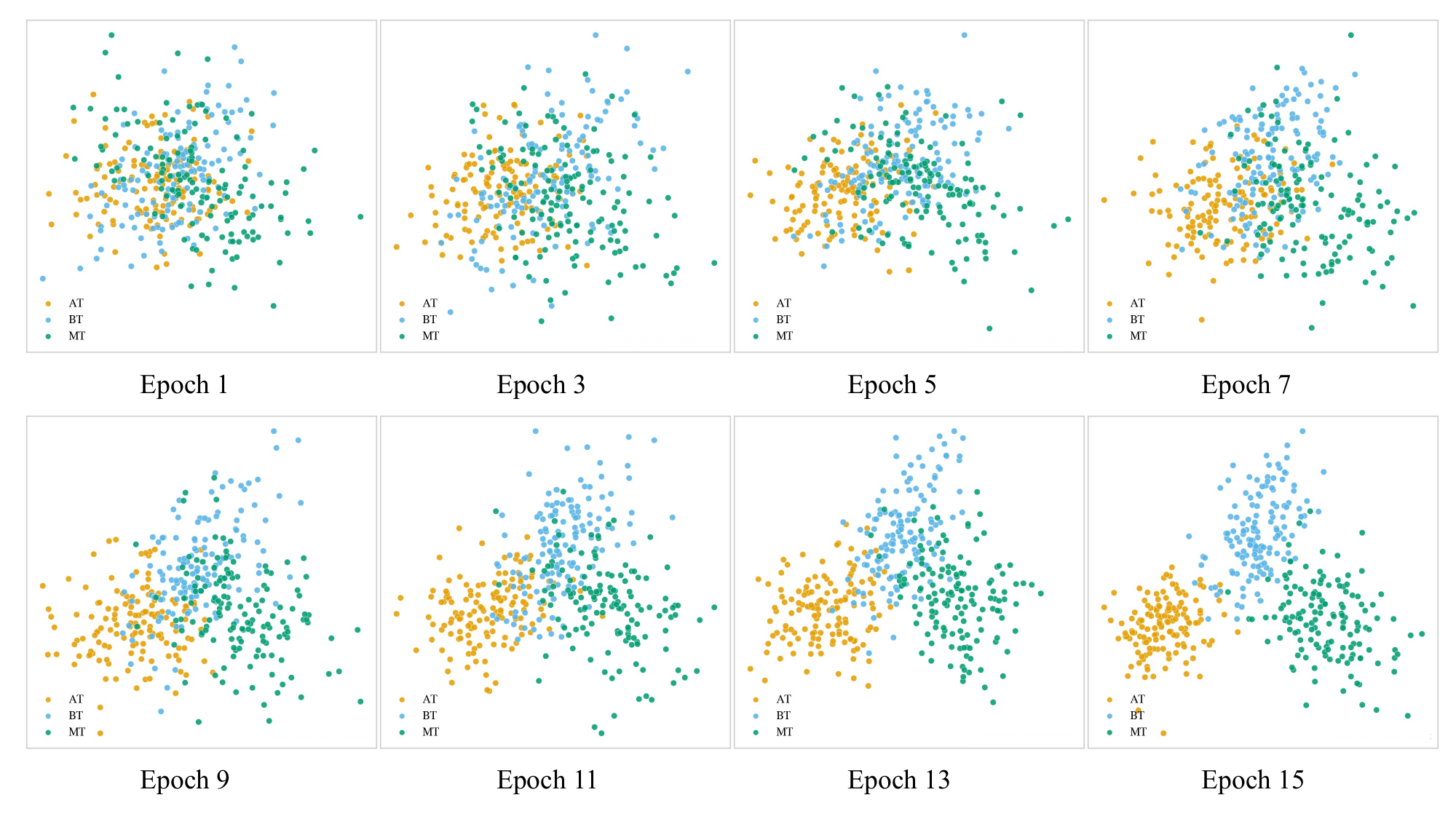}
\caption{\textbf{T-SNE before evidence grounding.} Slide representations are less separated before the GCE evidence objective is applied.}
\label{fig:tsne_before}
\end{figure}

\begin{figure}[!htbp]
\centering
\includegraphics[width=\columnwidth]{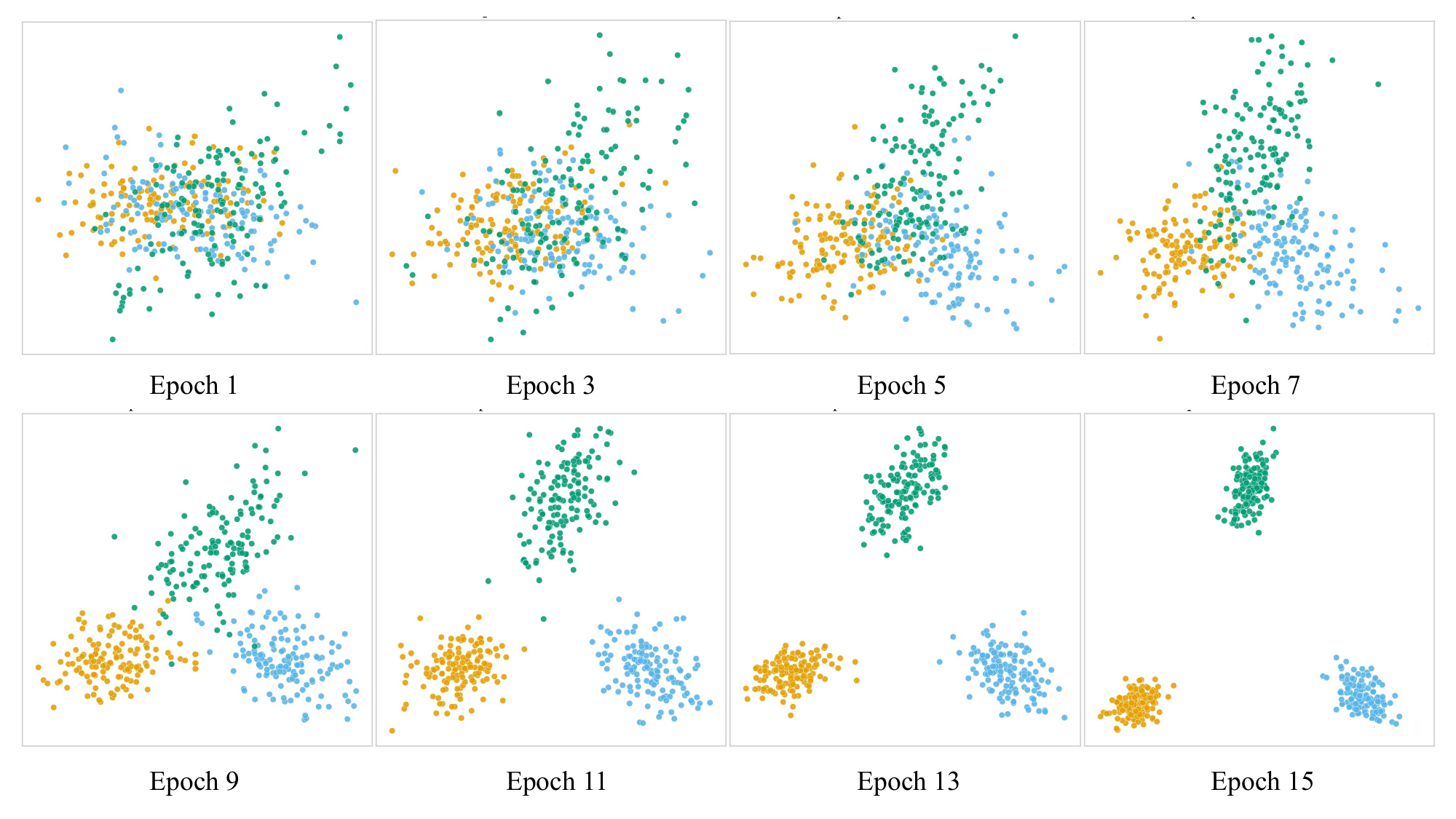}
\caption{\textbf{T-SNE after evidence grounding.} GCE training produces more separated slide representations, consistent with the classification and evidence diagnostics.}
\label{fig:tsne_after}
\end{figure}

Figures~\ref{fig:tsne_before} and \ref{fig:tsne_after} show that evidence-oriented training changes representation geometry.
Before grounding, class clusters remain partially entangled across training epochs; after GCE grounding, the late-epoch embeddings separate into more compact class-specific regions.
The continuous selector becomes bimodal during training (Figure~\ref{fig:motivation_main}, middle), enabling reliable discretization.
Appendix Figures~\ref{fig:mechanism_panels}--\ref{fig:mechanism_repair} and Figures~\ref{fig:attention_1}--\ref{fig:attention_6} provide additional mechanism diagrams and qualitative overlays.
These visualizations are secondary to the intervention metrics, but they help diagnose whether the learned selector behaves like a structured evidence mechanism rather than a post-hoc heatmap.

\subsection{Mechanism and Qualitative Overlays}
Figures~\ref{fig:tsne_before} and \ref{fig:tsne_after} report the available T-SNE visualizations.
Appendix Figures~\ref{fig:mechanism_panels}--\ref{fig:mechanism_repair} include the mechanism diagrams, and Figures~\ref{fig:attention_1}--\ref{fig:attention_6} include the qualitative attention/evidence overlays.
The mechanism panels are included to make the implementation pathway explicit: feature adaptation, anchor response, gate formation, noisy-OR utility, thresholding, and repair are separate operations.
The qualitative overlays serve as sanity checks rather than quantitative proof; they illustrate the same pattern measured in the intervention tables, namely that recovered evidence is compact while attention often remains diffuse.
\FloatBarrier
The mechanism diagrams below decompose the GCE-MIL pipeline into the operations that are compressed into Figure~\ref{fig:method_overview}.
Each mechanism is shown as a separate large figure so that the score curves and gate distributions remain readable.
The first three figures emphasize how patch features are adapted and compared with anchors; the last three emphasize how continuous gates become a recovered discrete subset.

\begin{figure}[H]
\centering
\includegraphics[width=0.86\textwidth]{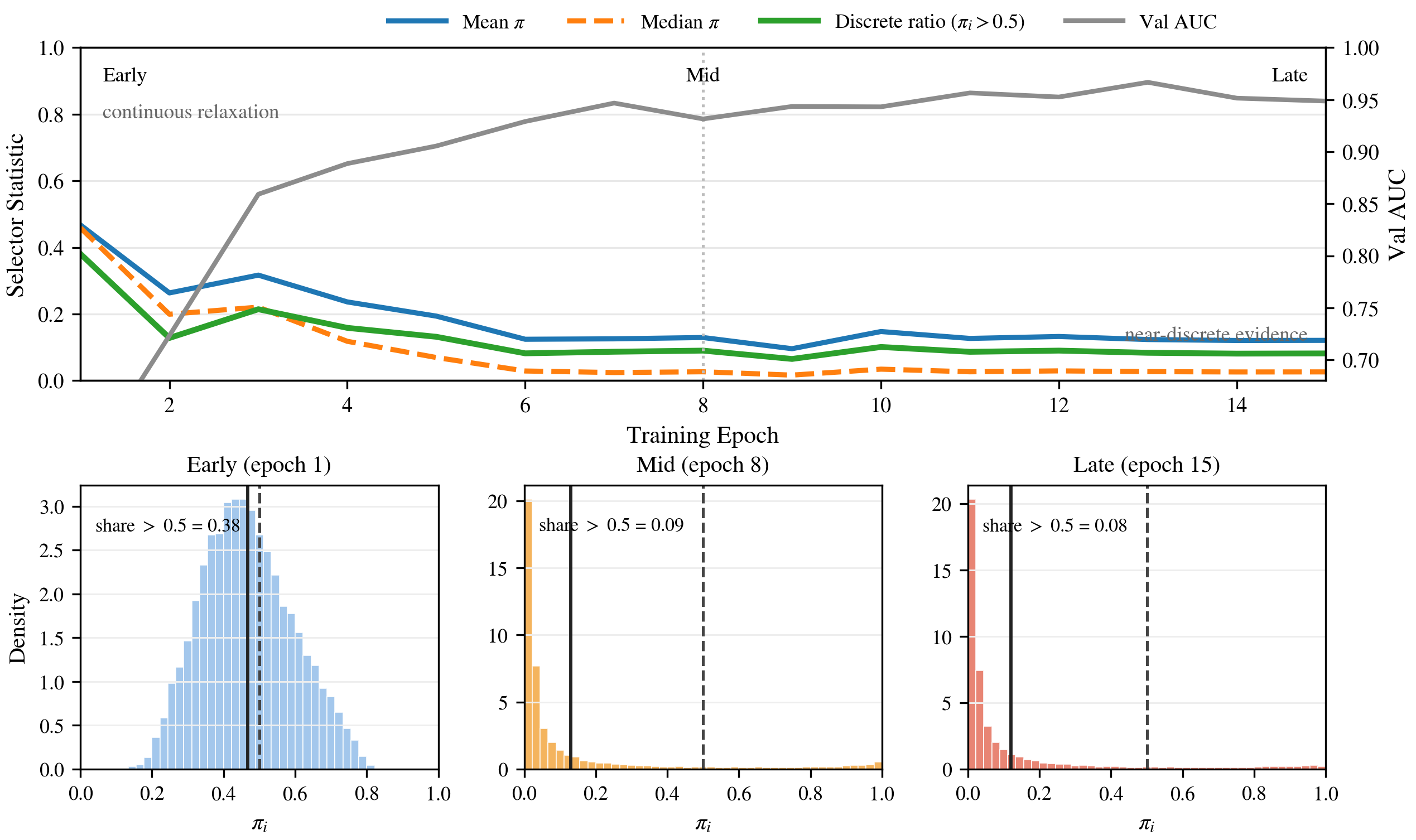}
\caption{\textbf{Mechanism 1: feature adaptation.} Low-rank residual adaptation keeps the pretrained feature space close to UNI while improving selector compatibility.}
\label{fig:mechanism_panels}
\end{figure}

\begin{figure}[H]
\centering
\includegraphics[width=0.86\textwidth]{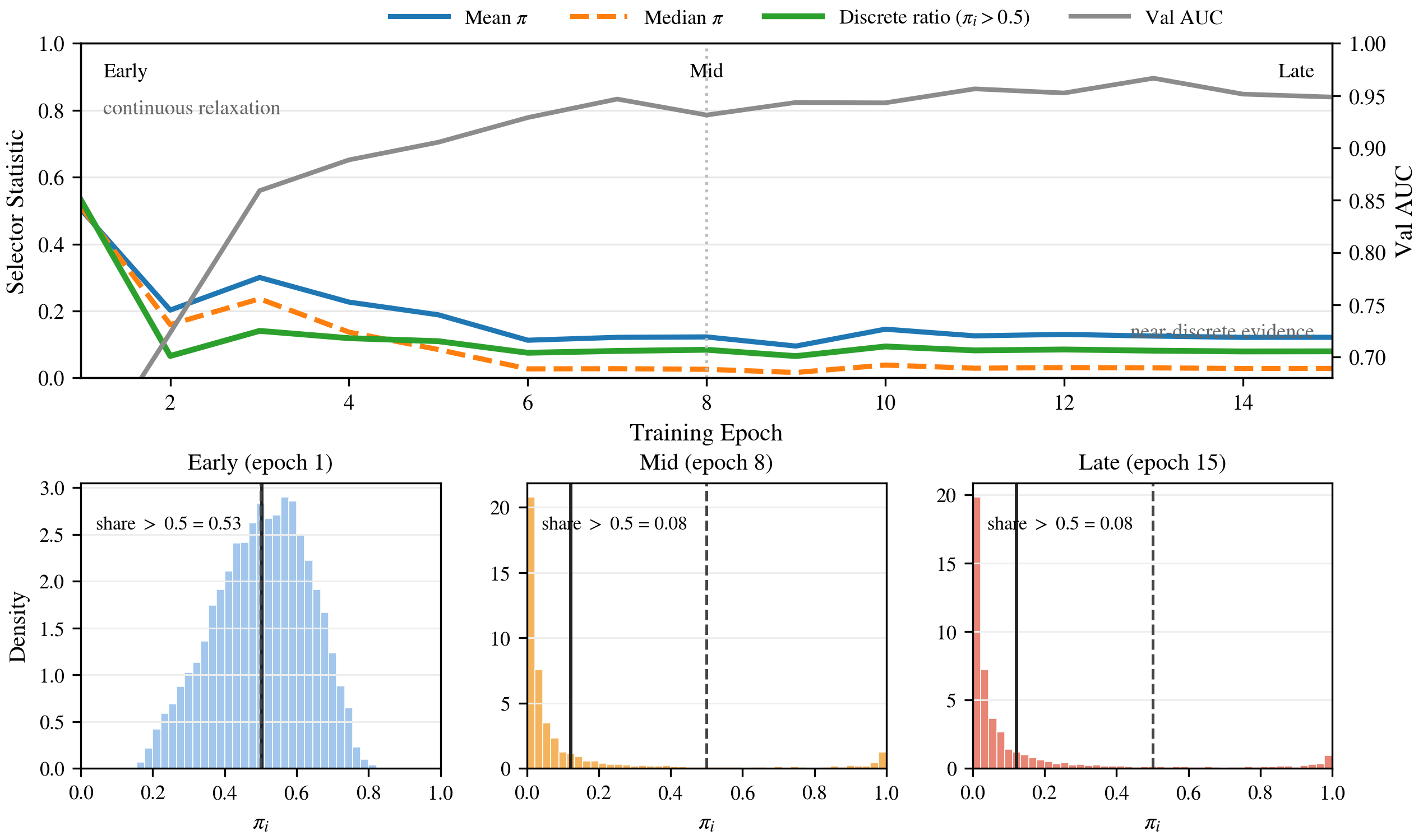}
\caption{\textbf{Mechanism 2: anchor response.} The bridge maps patch features into TITAN anchor space and produces patch-anchor responses used by the coverage utility.}
\label{fig:mechanism_anchor}
\end{figure}

\begin{figure}[H]
\centering
\includegraphics[width=0.86\textwidth]{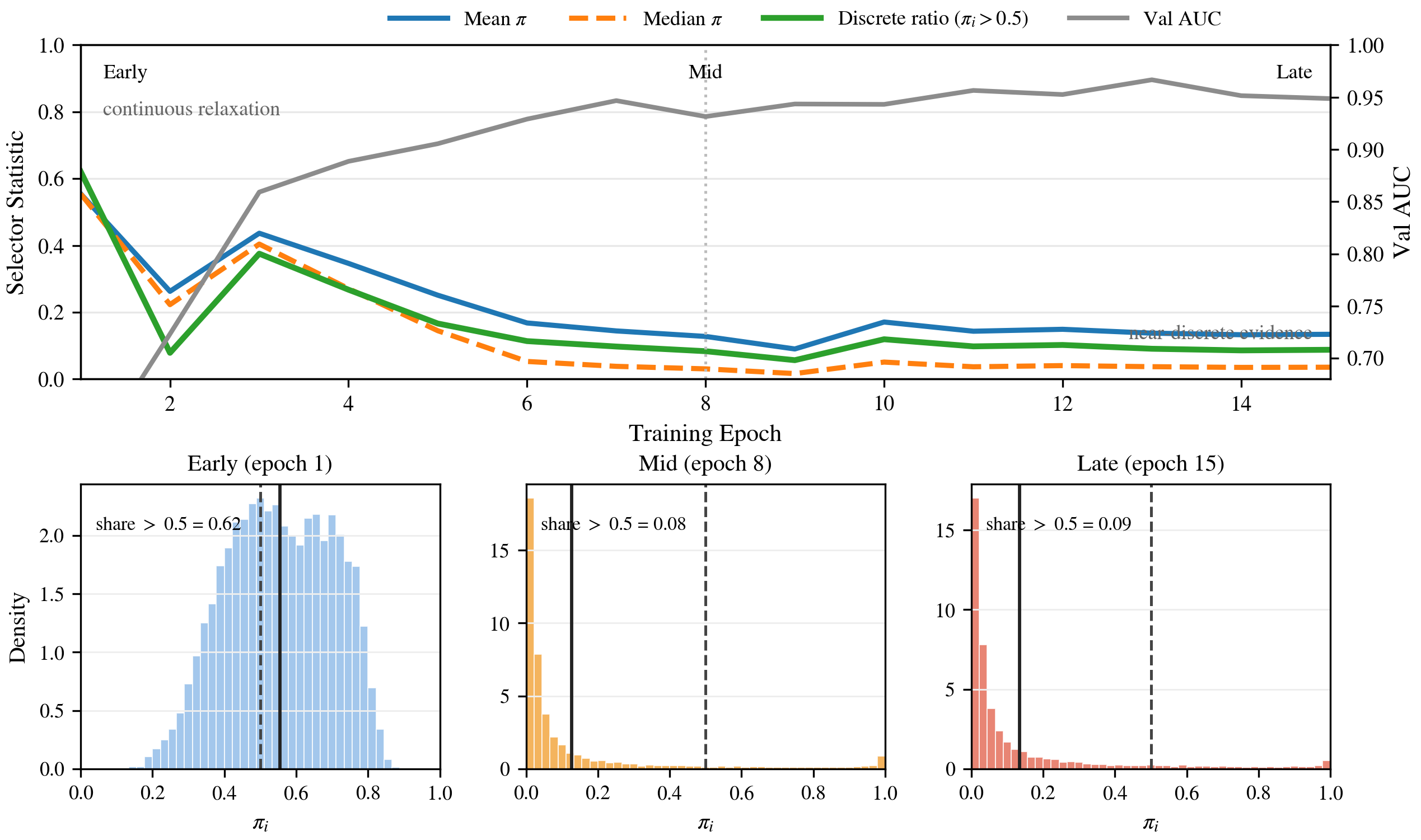}
\caption{\textbf{Mechanism 3: continuous gate.} Temperature annealing sharpens the selector distribution so that continuous gates can be recovered as a discrete subset.}
\label{fig:mechanism_gate}
\end{figure}

\begin{figure}[H]
\centering
\includegraphics[width=0.86\textwidth]{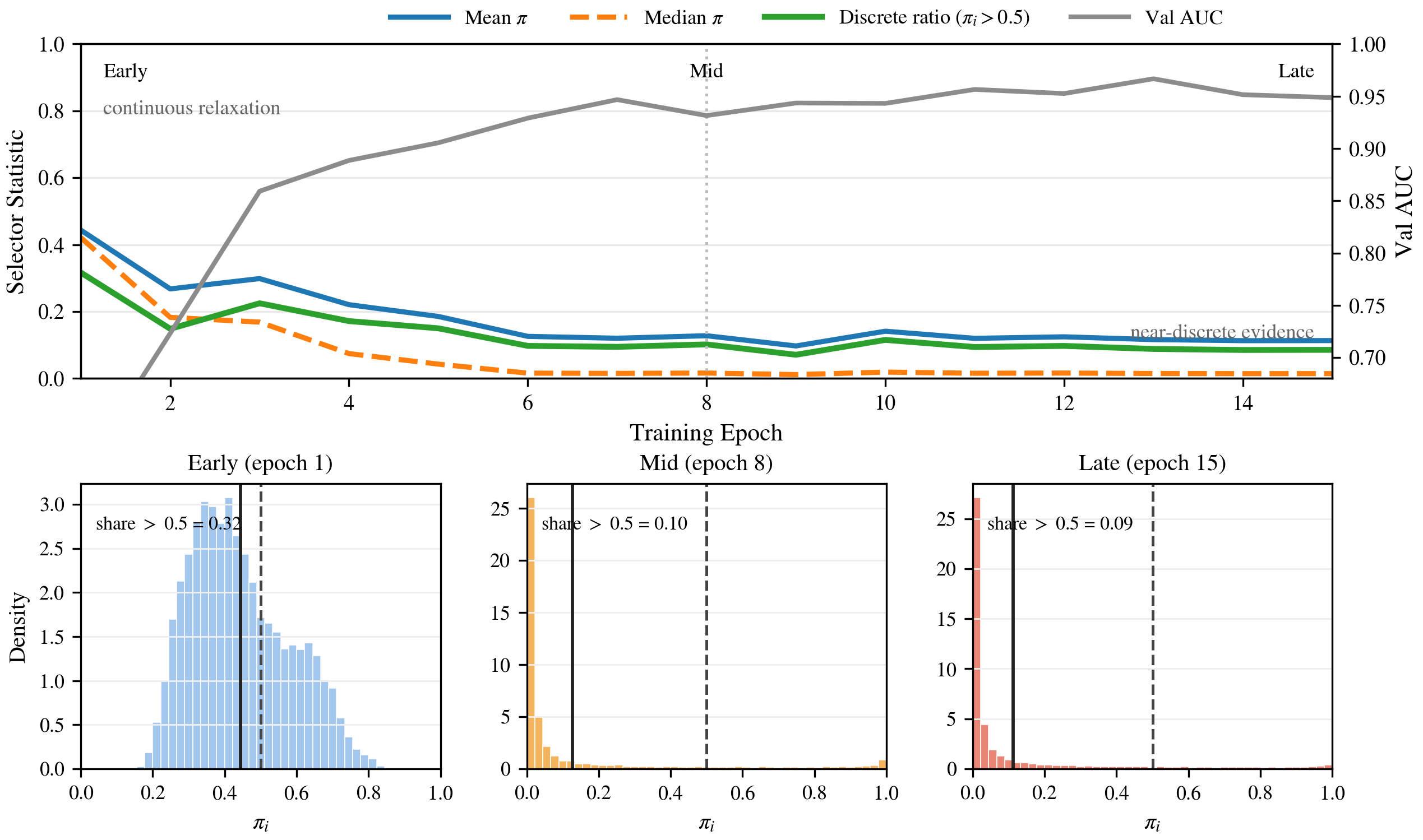}
\caption{\textbf{Mechanism 4: noisy-OR utility.} Exact noisy-OR coverage gives diminishing returns once an anchor is already covered, encouraging complementary evidence.}
\label{fig:mechanism_utility}
\end{figure}

\begin{figure}[H]
\centering
\includegraphics[width=0.86\textwidth]{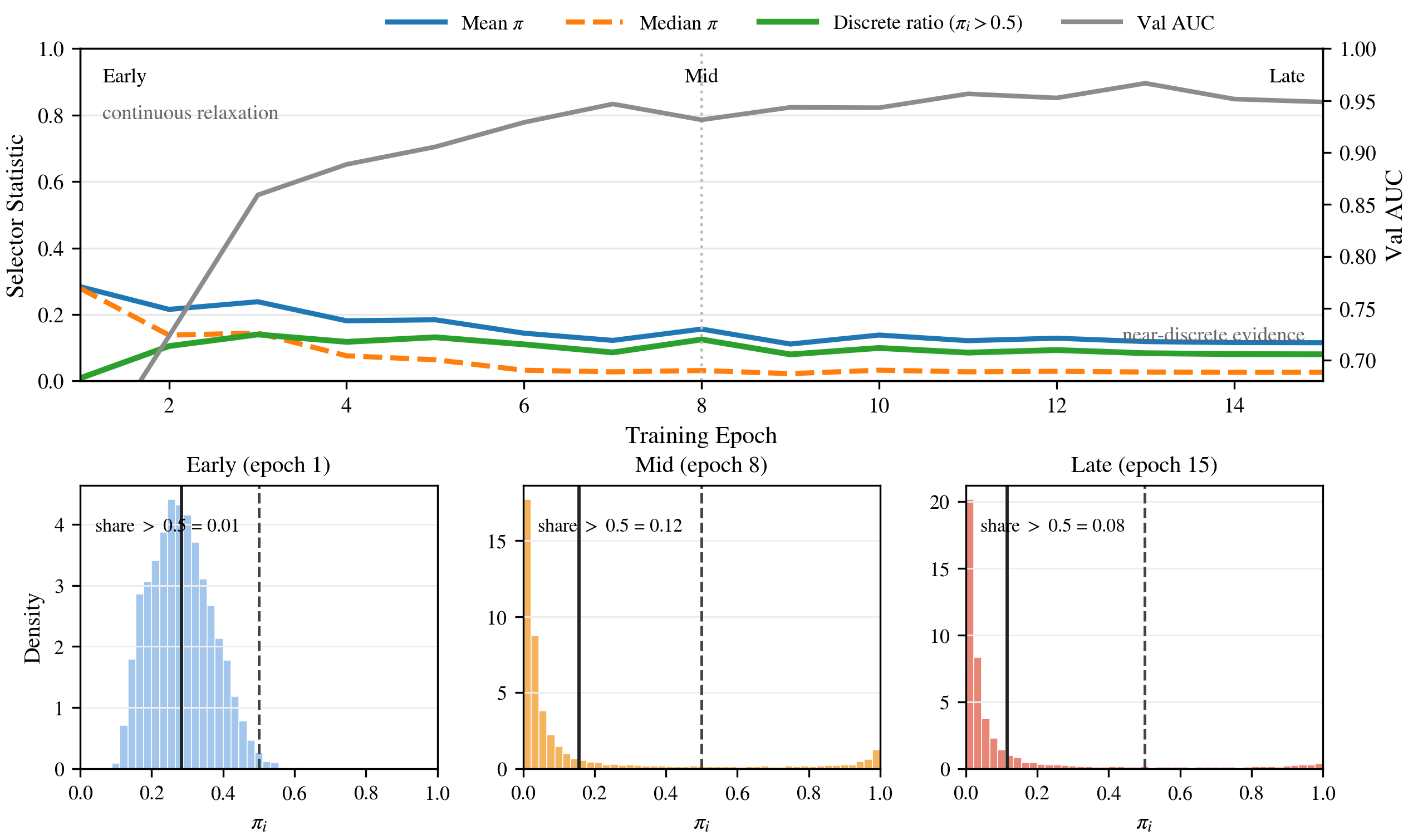}
\caption{\textbf{Mechanism 5: threshold recovery.} The initial discrete subset is obtained by thresholding the continuous gate and falling back to the top patch when needed.}
\label{fig:mechanism_threshold}
\end{figure}

\begin{figure}[H]
\centering
\includegraphics[width=0.86\textwidth]{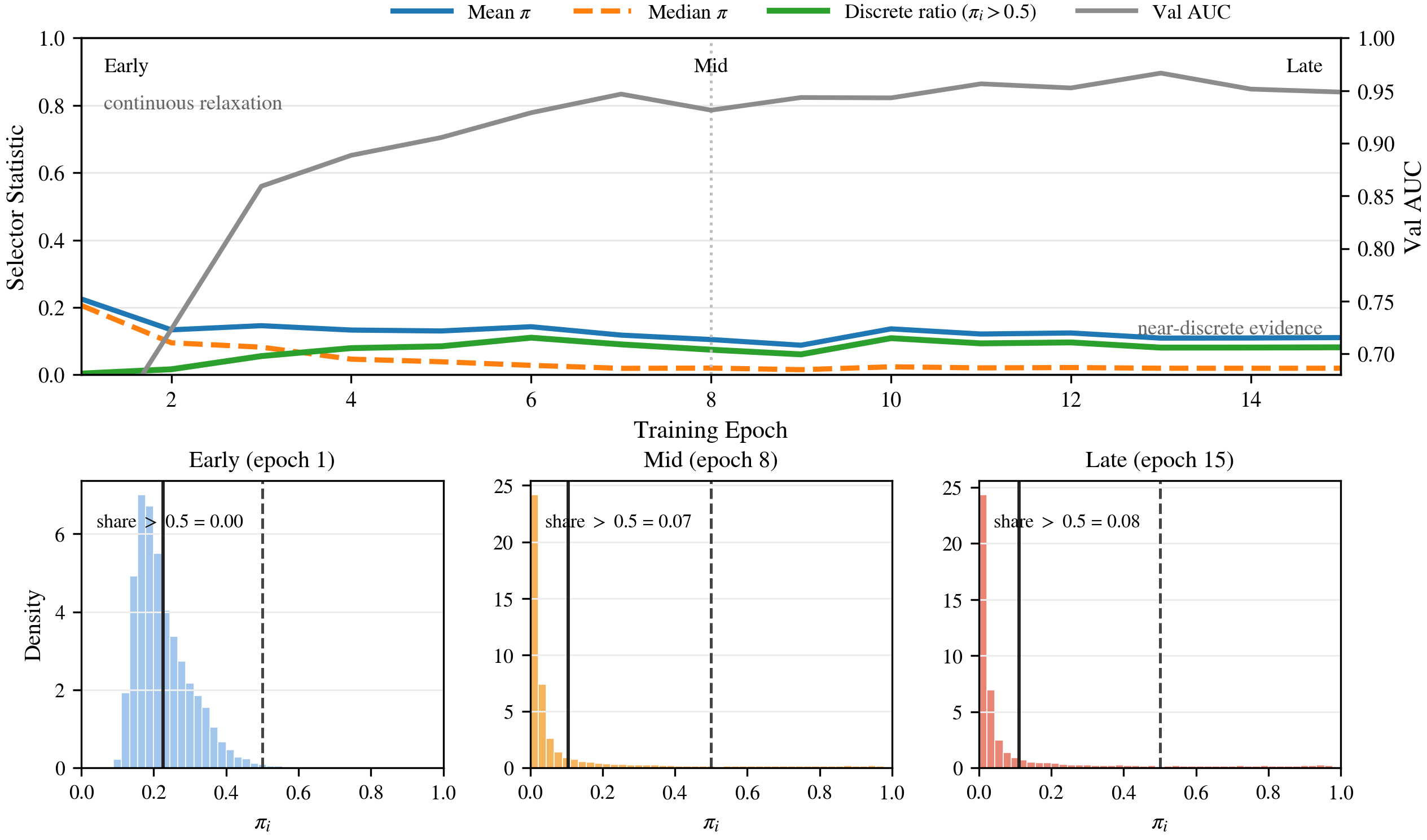}
\caption{\textbf{Mechanism 6: greedy repair.} Repair adds patches according to exact marginal utility until coverage and sufficiency criteria are restored.}
\label{fig:mechanism_repair}
\end{figure}

\FloatBarrier
The following qualitative overlays show representative evidence maps exported from the evaluation pipeline.
They are not used to claim pixel-level correctness; the quantitative support for evidence quality comes from keep-only, remove, C-D gap, and CAMELYON-16 localization metrics.
Instead, these examples help readers inspect whether the recovered subset is compact, whether it avoids obvious background/artifact regions, and how it differs from diffuse attention maps.

\begin{figure}[H]
\centering
\includegraphics[width=0.92\columnwidth]{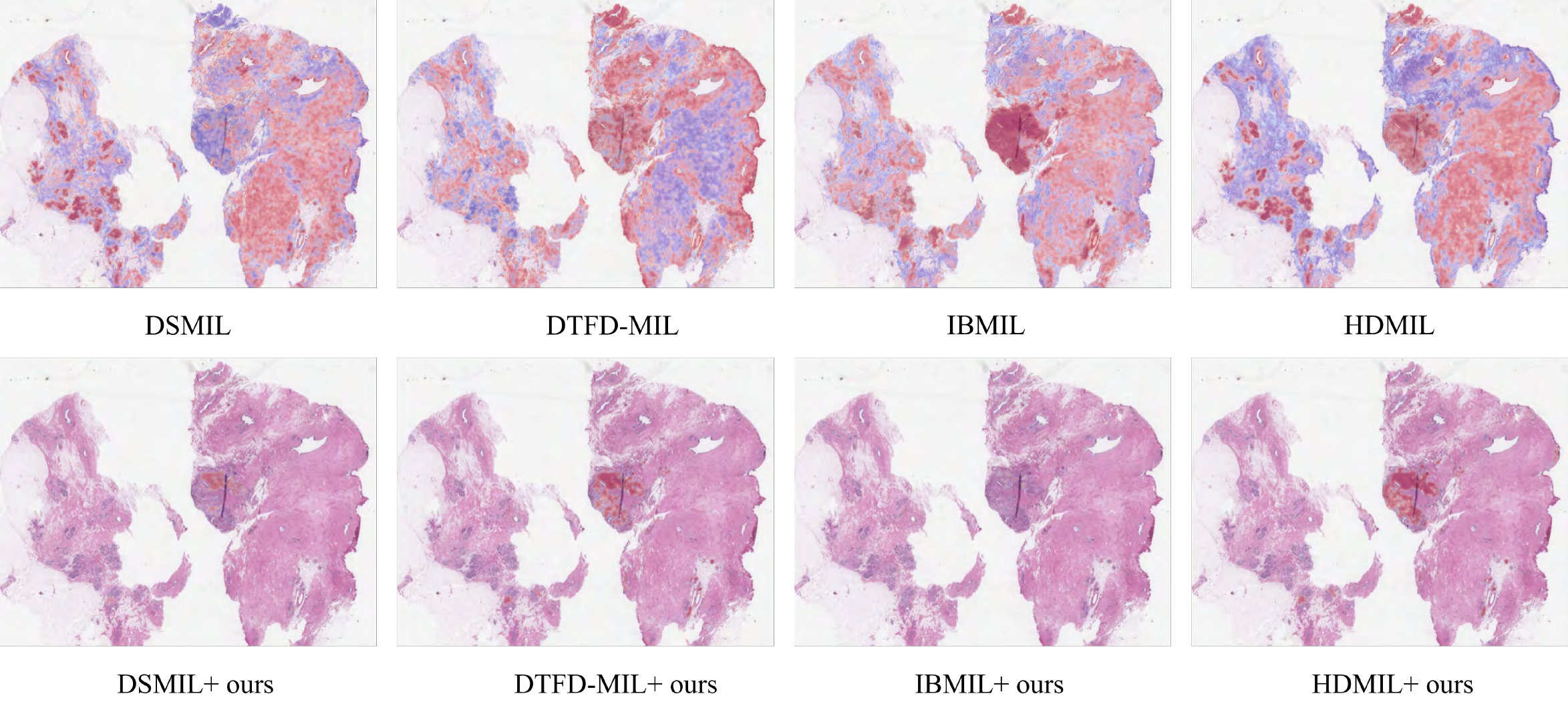}
\caption{\textbf{Qualitative evidence overlay 1.} Representative slide-level attention and recovered GCE evidence.}
\label{fig:attention_1}
\end{figure}

\begin{figure}[H]
\centering
\includegraphics[width=0.92\columnwidth]{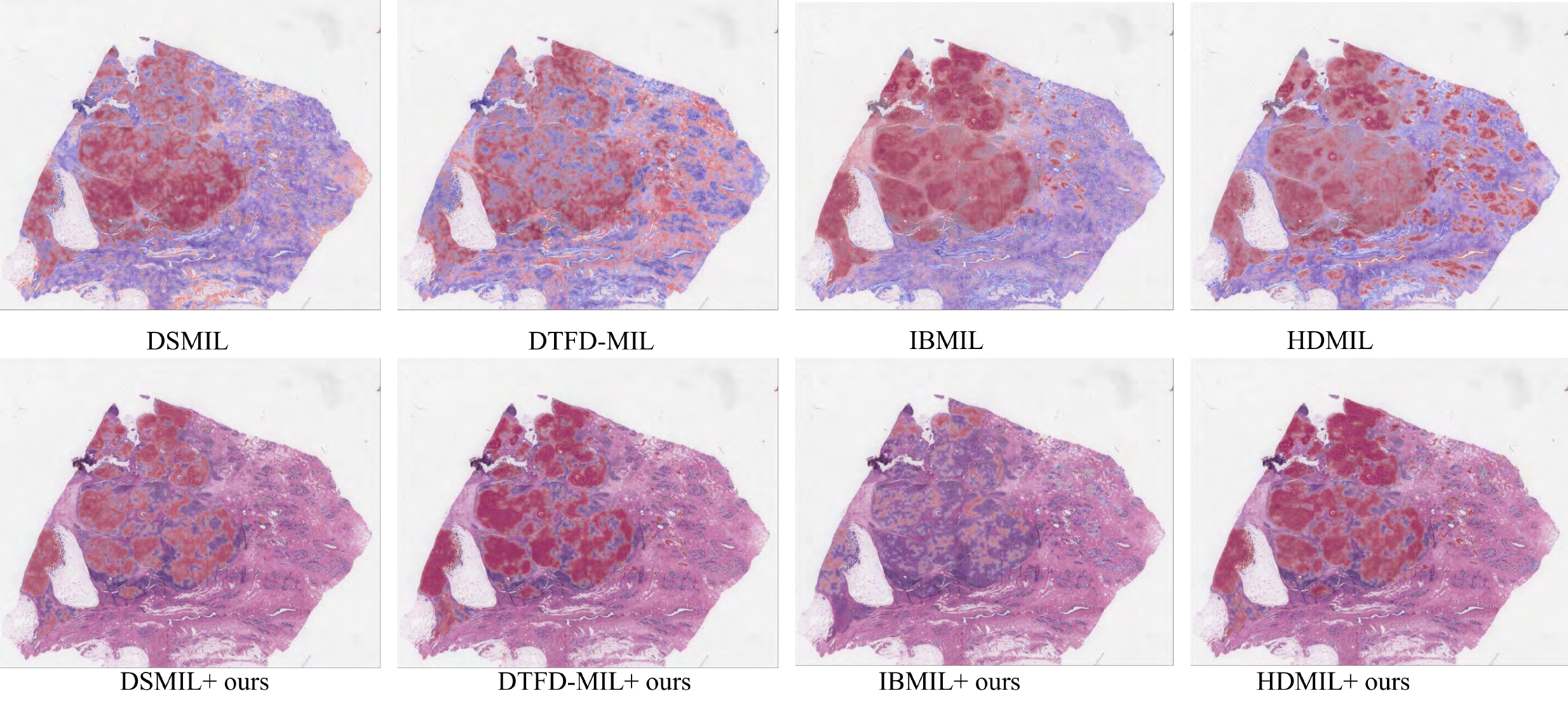}
\caption{\textbf{Qualitative evidence overlay 2.} Representative slide-level attention and recovered GCE evidence.}
\label{fig:attention_2}
\end{figure}

\begin{figure}[H]
\centering
\includegraphics[width=0.92\columnwidth]{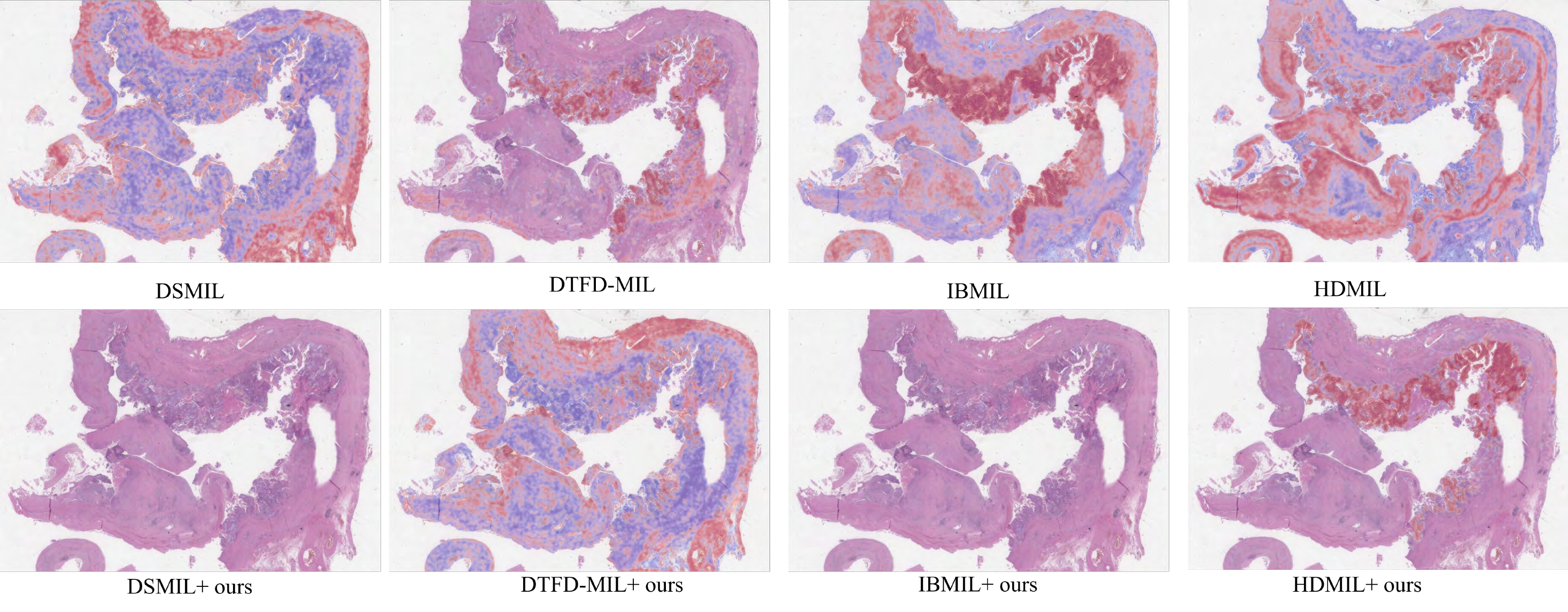}
\caption{\textbf{Qualitative evidence overlay 3.} Representative slide-level attention and recovered GCE evidence.}
\label{fig:attention_3}
\end{figure}

\begin{figure}[H]
\centering
\includegraphics[width=0.92\columnwidth]{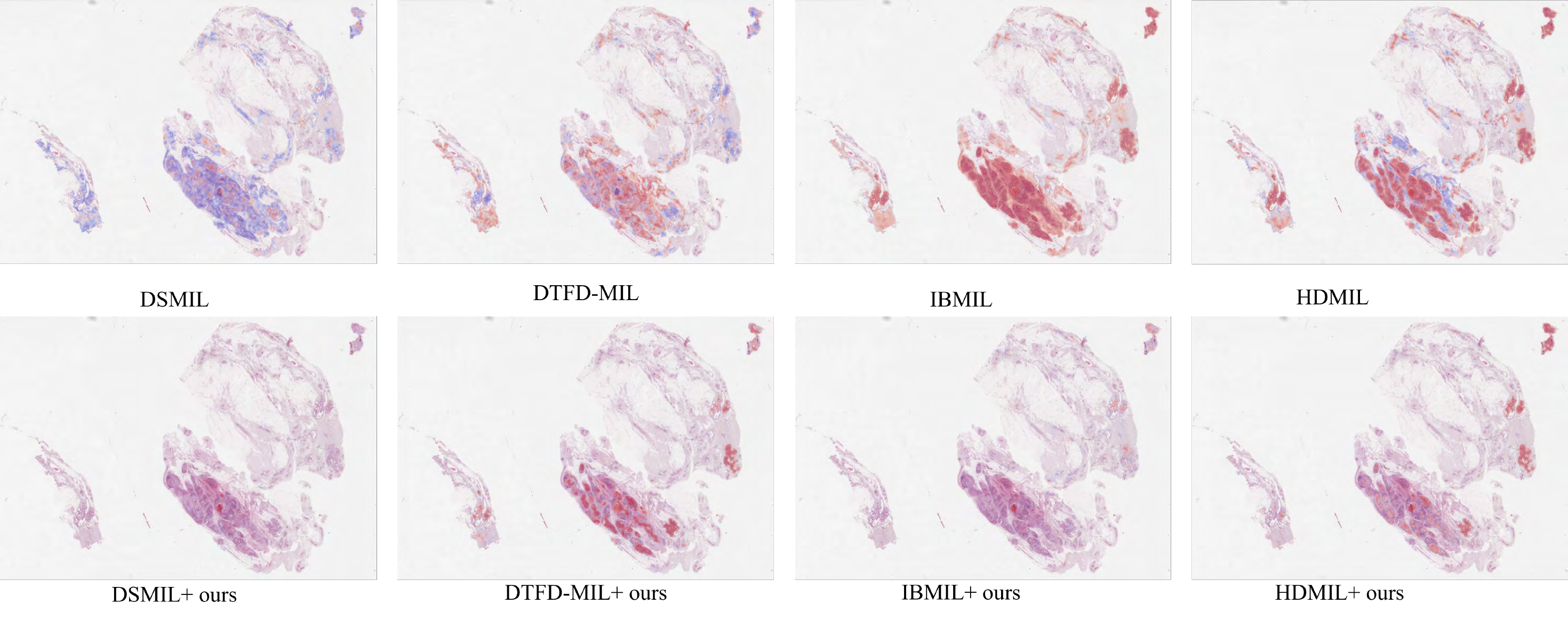}
\caption{\textbf{Qualitative evidence overlay 4.} Representative slide-level attention and recovered GCE evidence.}
\label{fig:attention_4}
\end{figure}

\begin{figure}[H]
\centering
\includegraphics[width=0.92\columnwidth]{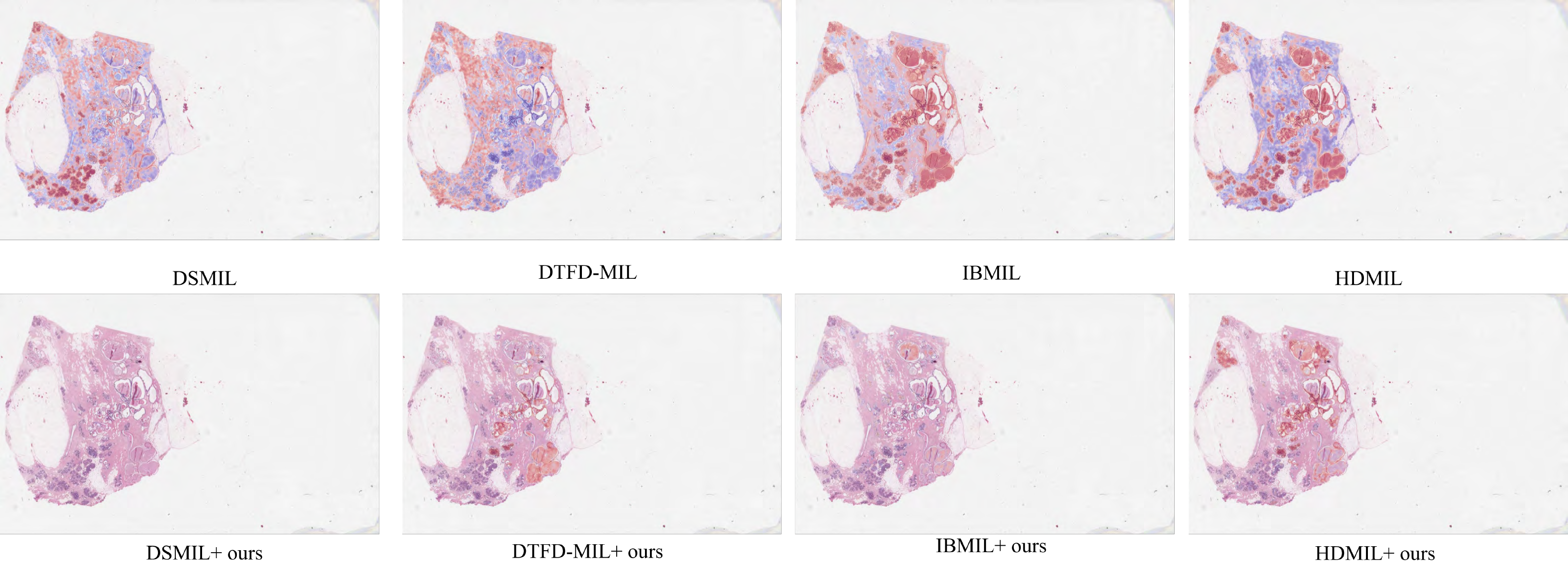}
\caption{\textbf{Qualitative evidence overlay 5.} Representative slide-level attention and recovered GCE evidence.}
\label{fig:attention_5}
\end{figure}

\begin{figure}[H]
\centering
\includegraphics[width=0.92\columnwidth]{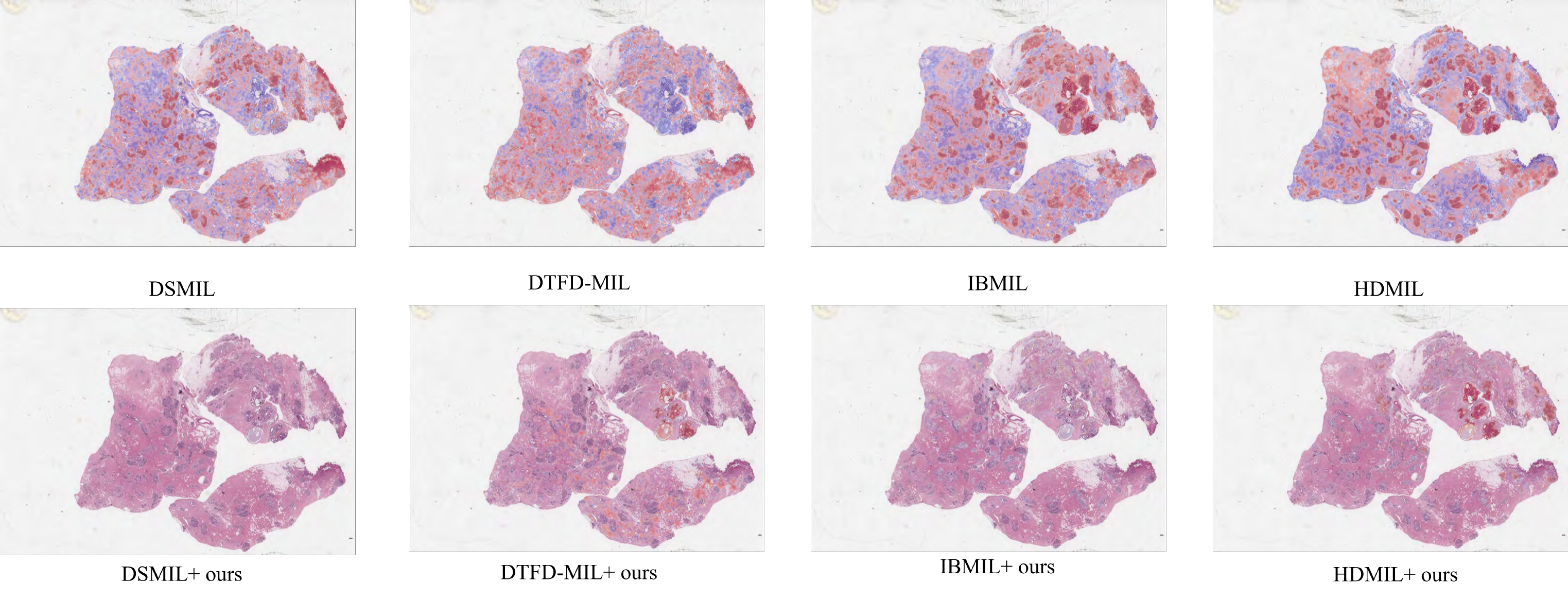}
\caption{\textbf{Qualitative evidence overlay 6.} Representative slide-level attention and recovered GCE evidence.}
\label{fig:attention_6}
\end{figure}

\FloatBarrier
\section{Limitations and Discussion}
\label{app:limitations}
The S/N/R formalization operates at patch level; extending it to pixel-level or region-level evidence remains open.
The anchor bank uses fixed text prompts that may not transfer to rare or previously unseen tissue types without prompt adaptation.
Finally, S/N/R measures model-relative evidence faithfulness, not pathologist-verified causal mechanisms---bridging this gap requires clinical validation studies.

\end{document}